\documentclass[journal]{IEEEtran}

\ifCLASSINFOpdf
\else
\fi

\usepackage{times}

\usepackage[numbers]{natbib}
\usepackage{multicol}
\usepackage[bookmarks=true]{hyperref}

\usepackage{dsfont}
\usepackage{amsmath}
\usepackage{amsfonts}
\usepackage{amsthm}
\usepackage{amssymb}
\usepackage{physics}
\usepackage{graphicx}

\DeclareMathOperator*{\argmin}{arg\,min}

\usepackage{balance}

\usepackage[caption=false,font=footnotesize]{subfig}

\usepackage{algorithm}
\usepackage{algpseudocode}

\theoremstyle{plain}
\newtheorem{thm}{Theorem}
\newtheorem{problem}{Problem}

\newtheorem{corollary}{Corollary}

\newtheorem{assum}{Assumption}

\newcommand{\bu}{\mathbf{u}}
\newcommand{\bx}{\mathbf{x}}
\newcommand{\bz}{\mathbf{z}}
\newcommand{\bxi}{\mathbf{\xi}}

\newcommand{\belief}{\mathbf{b}}
\newcommand{\bs}{\mathbf{s}}
\newcommand{\bm}{\mathbf{m}}
\newcommand{\bn}{\mathbf{n}}

\usepackage{acronym}
\acrodef{MPC}[MPC]{Model Predictive Control}
\acrodef{POMDP}[POMDP]{Partially Observable Markov Decision Process}
\acrodef{EKF}[EKF]{Extended Kalman Filter}
\acrodef{iLQG}[iLQG]{iterative linear-quadratic Gaussian control}
\acrodef{iLQR}[iLQR]{iterative linear-quadratic regulator control}
\acrodef{DDP}[DDP]{Differential Dynamic Programming}

\begin{document}

\title{Stochastic Dynamic Games in Belief Space}

\author{Wilko~Schwarting,
        Alyssa~Pierson,
        Sertac~Karaman, 
        and~Daniela~Rus%
\thanks{
This work is supported in part by NSF Grant 1723943, the Office of Naval Research (ONR) Grant N00014-18-1-2830, and Toyota Research Institute (TRI). TRI provided funds to assist the authors with their research but this article solely reflects the opinions and conclusions of its authors and not TRI or any other Toyota entity. (Corresponding author: Wilko Schwarting.)

W. Schwarting, A. Pierson, and D. Rus are with the Computer Science  and Artificial Intelligence Laboratory (CSAIL), Massachusetts Institute of Technology, Cambridge, MA 02139 USA (e-mail: \{wilkos, apierson, rus\}@csail.mit.edu).

S. Karaman is with the Laboratory of Information and Decision  Systems (LIDS), Massachusetts Institute of Technology, Cambridge, MA 02139 USA (e-mail:sertac@mit.edu).}%
}

\maketitle

\begin{abstract}
Information gathering while interacting with other agents under
sensing and motion uncertainty is critical in domains
such as driving, service robots, racing, or
surveillance. 
The interests of agents may be at odds with
others, resulting in a stochastic non-cooperative dynamic game.
Agents must predict others' future actions without
communication, incorporate their
actions into these predictions, account for uncertainty and noise in information gathering, and consider what information their actions reveal.
Our solution uses local iterative dynamic programming in Gaussian belief space to solve a game-theoretic continuous POMDP.
Solving a quadratic game in the backward pass of a game-theoretic belief-space variant of iLQG achieves a runtime polynomial in the number of agents and linear in the planning horizon.
Our algorithm yields linear feedback policies for our robot, and predicted feedback policies for other agents.
We present three applications: active surveillance, guiding eyes for a blind agent, and autonomous racing.
Agents with game-theoretic belief-space planning win 44\% more races than without game theory and 34\% more than without belief-space planning.
\end{abstract}

\begin{IEEEkeywords}
Motion and Path Planning, Optimization and Optimal Control, Multi-Robot Systems, Game-Theoretic Planning
\end{IEEEkeywords}

\IEEEpeerreviewmaketitle

\section{Introduction}
\label{sec:introduction}
\IEEEPARstart{W}{e} aim to develop planners for multi-agent systems that are robust under uncertainty and combine information-seeking behavior with game-theoretic reasoning. 
While game theory can  model  the  interaction  and
dependency  among  agents,  it  does  not  address  the  quality  of
the information available to the agent for decision making.
Agents must plan and act within a game, remain robust to uncertainty, gain information, and leverage the information gain to improve their control policies.
We propose an approach that combines game-theoretic planning with belief-space planning, leveraging the interaction models from game theory while incorporating uncertainties in the modeled dynamics and perception. 
In multi-agent systems, we find that agents gather information to reduce uncertainty while maintaining decision-making strategies that support complex interactions.
Applications include assistive robotics, surveillance, pursuer-evader games, and racing.

\begin{figure}
    \centering
        \includegraphics[width=1.0\columnwidth]{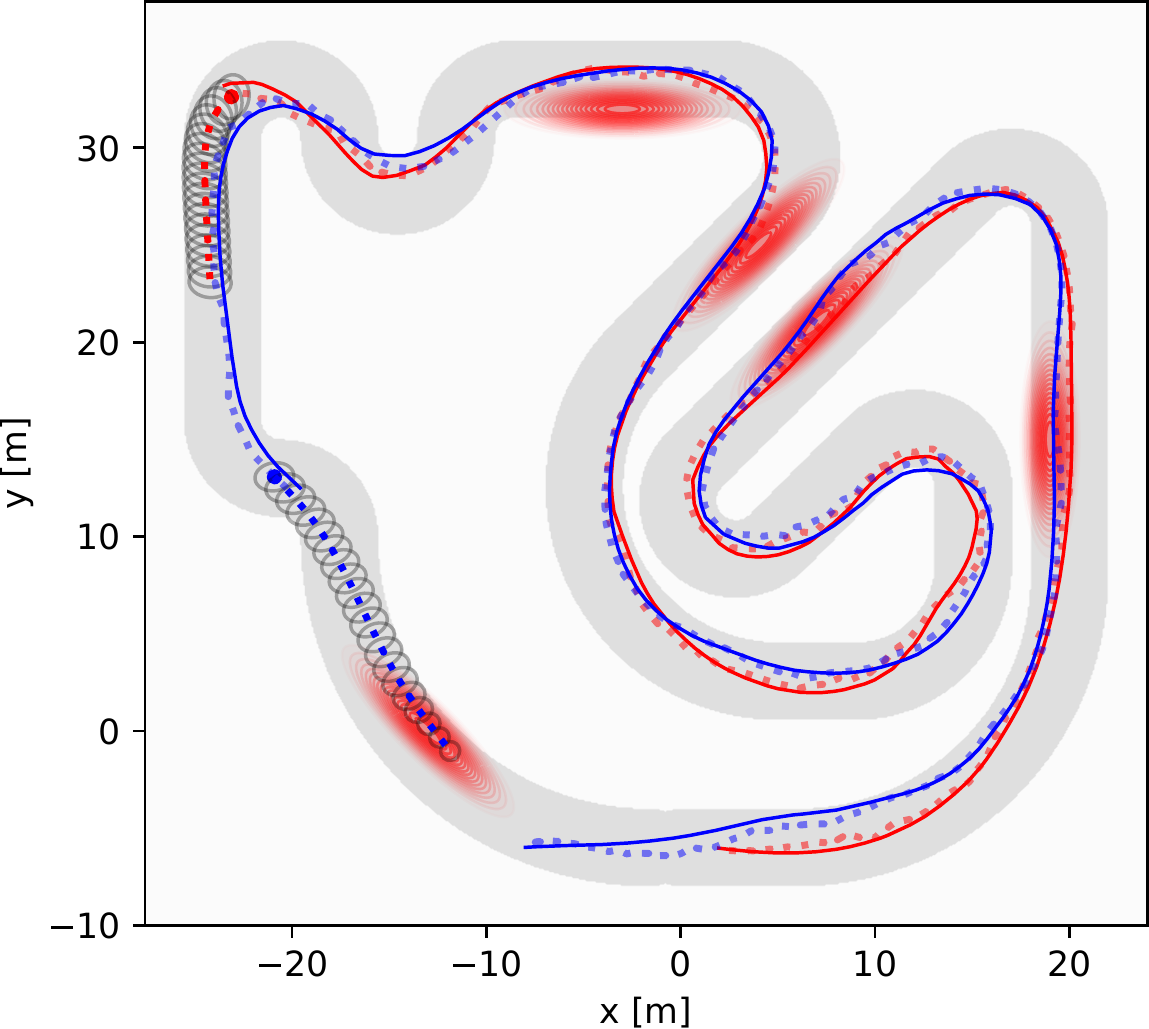}
    \caption{One application we present is dynamic racing. Here, the blue agent starts with a disadvantage but is equipped with better acceleration and capable of moving faster through corners than the red agent. Our approach allows the blue agent to overtake and win the race. Planned trajectories and chance constraints are shown in dashed lines and ellipses. The traces correspond to the true state (solid) and the noisy EKF-estimate (dashed) available to each agent during the race. The red areas are zones with low noise observations and reduce uncertainty.}
    \label{fig:race_big}
\end{figure}

While each agent operates independently and does not reveal plans or intentions through communication, agents have approximate models of the other agents. Dynamic models allow us to infer the ability of other agents to move in the environment, approximate cost models encode the agents' objectives. Models of how other agents perceive the world from observations allow estimating what other agents know and with what certainty. Our work is related to model predictive receding horizon planners: agents start from an initial belief about themselves, others, and the environment and imagine how the future will evolve if they and other agents were to execute certain actions. These models can be prescribed, observed, sensed, or communicated. As plans are not shared among agents, we propose a game-theoretic framework that predicts the interactive policies of other agents while optimizing for our own policy.

Within the game-theoretic framework, agents take actions that increase their information gain, which in turn results in the ability to improve their control policies with reduced uncertainty. For example, an assistive robot tasked with guiding a human may explore the environment to reduce uncertainty and better navigate. Conversely, a game-theoretic setting can model adversarial agents. Here, agents may choose to \emph{hide} to prevent others from gathering information about themselves, which is relevant to surveillance applications. In the context of racing, agents may force others to increase their uncertainty, such as by pressuring them to drive too fast in a corner which increases uncertainty in their state, or by simply pushing them into the dark. It is therefore not only important to reason about a robot's own uncertainty but also the uncertainty of other agents in the environment, and even more so how one's actions impact the change in the uncertainty of others.

Game-theoretic models have not only proven useful to model interactions between autonomous systems, but also in integrating interactive human predictions into autonomous decision-making and planning. We can model the actions of humans as expected cost-minimizing and estimate human cost functions from past observed trajectories with Inverse Reinforcement Learning (IRL)~\cite{ziebart2008maximum}. Consequently, computing expected cost-minimizing actions based on the learned cost functions generates human predictions. The expected cost-minimizing behavior can also be interpreted as the best response to an autonomous agent's actions. This best response setting allows us to estimate how the autonomous system's actions influence human actions. The autonomous system can therefore implicitly control the human's actions to a certain degree. This technique has been applied to predict interactive human behavior for autonomous vehicles \cite{sadigh2016planning, kuderer2015learning, schwartingSVO}, and to predict pedestrians~\cite{kretzschmar2014learning}. The combination of game-theoretic modeling of human behavior and information-seeking planning is therefore even more promising.

For instance, home service robots can provide assistance and support to humans, particularly the elderly population. These robots need to work near humans, gauge the human's intent, and understand the state of mind of others to better perform tasks. They have to avoid confusion and misunderstandings and will need to seek information about both their environment and surrounding humans. Additionally, these autonomous systems need to also reason about the amount of information and understanding the human has about the robot. The robot can aid the human's understanding through explicit communication, as well as implicitly through behavior, such as moving to visible locations or indicating intent by unambiguously moving in the desired direction.

We propose a solution that combines multi-agent game-theoretic decision-making under uncertainty and belief-space planning (BSP). Our approach supports robust solutions to a wide range of multi-robot applications in which dealing with uncertainty, the need to gather information, and game-theoretic decision-making are fundamental.
We build on important advancements in two areas: game-theoretic planning and belief-space planning. Game-theoretic planning successfully solves problems where an agent's objective is at odds with the objective of other agents, such as in modeling human behavior in traffic~\cite{li2018game, schwartingSVO, schwarting2018planning}, and leveraging the effects on humans by autonomous cars~\cite{sadigh2016planning}. \cite{marden2018game} gives a recent review on game theory and control. In game theory, the Nash equilibrium is a proposed solution of a non-cooperative game involving two or more players. Each player is assumed to know the equilibrium strategies of the other players, and no player has anything to gain by changing only their own strategy. Solving for Nash equilibria has been applied to competitive racing~\cite{williams2018best,liniger2017non,spica1801real} and guiding vehicles through intersections \cite{dreves2018generalized}. Solution methods include Iterated Best Response \cite{williams2018best,spica1801real, Wang2019, Wang2020}, iterative quadratic approximations~\cite{fridovich2020efficient, cleac2019algames}, using discrete payoff matrices \cite{liniger2017non}, or solving the necessary conditions~\cite{schwartingSVO}.
We will solve the necessary condition of a static quadratic game at each stage in the backward-pass of \ac{iLQG} to solve for the Nash equilibrium of the dynamic game.

While game-theoretic planning models the interaction and dependency among agents, it does not address the quality of information available to the agent for decision making. Belief-space planning \cite{kaelbling1998planning} uses beliefs, which are the distribution of the robot's state estimate, to represent the uncertainties in the perception of the robot. The problem of computing a control policy over the space of belief states is formally described as a Partially Observable Markov Decision Process (POMDP), and has been studied extensively. Solutions to POMDPs are known to be very complex. Solving a POMDP to global optimality is NP-hard: solutions such as point-based algorithms \cite{bai2014integrated, kurniawati2008sarsop, pineau2003point, hansen2004dynamic} in discrete space are bound to the curse of history, as well as sampling based solvers \cite{bry2011rapidly, patil2012estimating, prentice2009belief}. Optimization-based approaches have been developed for planning in continuous belief space~\cite{lee2013sigma, patil2015scaling, platt2010belief, van2012motion, sun2016stochastic}, by approximating beliefs as Gaussian distributions and computing a value function valid in local regions of the belief space. 
Similarly to \cite{van2012motion, sun2016stochastic}, we avoid the common maximum-likelihood observation assumption \cite{platt2010belief, patil2015scaling}.
In comparison to point-based algorithms which scale exponentially in the planning horizon $l$, optimization-based methods scale linearly, $\mathcal{O}(l)$. 

\subsection{Main Assumptions}

To generate reasonable predictions about other agents, we build on approximate prior information. Consider the analogy of a race car driver. A driver knows that other race cars will have comparable driving characteristics, while different classes of vehicles, like trucks, will have different handling dynamics. They also know that the other racing drivers desire to go as fast as possible around the track to win the race, without crashing into other vehicles or sliding off the road, similar to a cost model.
Lastly, they have experience in how other drivers observe the track and that the quality of perception decreases in the dark. For our autonomous systems, Assumption~\ref{assum:CommKnow} describes that we assume \emph{common knowledge} of models about the world.
\begin{assum}[Common Knowledge] \label{assum:CommKnow}
Agents have models for cost, dynamics, and observations of other agents.
\end{assum}
Related game-theoretic works, with applications ranging from pedestrian-robot interactions to autonomous racing, make similar assumptions by either prescribing dynamics and cost models~\cite{li2018game, williams2018best, liniger2017non, spica1801real} or learning cost models through IRL~\cite{sadigh2016planning, kretzschmar2014learning, schwartingSVO}.
Assumption~\ref{assum:CommKnow} allows agents to imagine how the future will evolve: If the robot and other agents were to execute given policies, how would model-based predictions of motions and observations in the world impact the beliefs over time? 
We do not assume any form of direct communication between agents and therefore do not have access to the policies of other agents. Instead, we predict interactive policies of other agents through game-theory and leveraging the models of costs, dynamics, and observations. The robot, as part of the game, can then leverage the influence of its actions on the predicted actions of other agents to their advantage.
We will show in a competitive racing example that Assumption~\ref{assum:CommKnow} can be relaxed in practice and that approximate models of other agents prove sufficient to improve performance.

We refer to beliefs as distributions over states and draw inspiration over how we design our system from the cognitive theory of mind. The cognitive theory of mind~\cite{scassellati2002theory} defines the ability to attribute mental states, such as beliefs, intents, or desires to oneself and others. It is integral to understanding that others have beliefs that are different from one's own.
Single-agent planning in belief space, reasoning about the uncertainty of only the own state, is limited to zero-order beliefs (e.g. I think...).
In contrast, we will also reason about the uncertainty of other agents. The theory of mind refers to this as first-order belief spaces (e.g., I think they think...).
Higher-order beliefs such as second-order belief spaces (e.g., I think they think that I think...) are beyond the scope of this paper as they quickly become computationally intractable by essentially defining beliefs over beliefs. We find that parametrizing belief spaces efficiently is essential to generating real-time capable algorithms. Assumption~\ref{assum:FirstOrder}, keeps computation complexity at a reasonable level and avoids an explosion in parameters in the recursive beliefs over beliefs. 
\begin{assum}[First-Order Beliefs]
\label{assum:FirstOrder}
Planning and prediction are limited to first-order beliefs: any robot $i$’s belief over another agent $j$ is the same as that agent $j$’s belief about themselves.
\end{assum}
In Section~\ref{sec:case-studies}, we evaluate cases, such as competitive racing, where this assumption is a simplification of the true system dynamics. In the racing scenario, all agents execute separate instances of our algorithm and therefore maintain separate beliefs. Thus, an agent's belief about themselves does not necessarily match the beliefs that others have about them. However, while these belief mismatches may occur, we see performance improvements over a game-theoretic baseline without belief-space planning, see Section~\ref{sec:autonomous_racing}, which highlights the importance of accounting for uncertainty and information gain in competitive racing and other applications.

The purpose of Assumptions~\ref{assum:CommKnow} and~\ref{assum:FirstOrder} is to enable interactive predictions of other agents in belief space while maintaining computational tractability. Since our approach is executed continuously in a receding horizon fashion, and we compute policies that are reactive to deviations from the predicted beliefs, the proposed method can adapt if the observed behavior differs from the predicted ones. Our approach continues to successfully control the agent under reasonable violations of the presented assumptions,
such as if the dynamics, observation, or cost models are inaccurate, if their own beliefs do not exactly match the beliefs of others, or if the other agent's optimization is sub-optimal.

\subsection{Contributions}

We present a computationally-tractable solution to multi-agent planning that combines game-theoretic planning and belief-space planning to interact within a problem formulated as a game, gain information, and leverage the information gain to improve the agents' control policies. The main limiting factor in applying either game theory or belief-space planning, and even more so the combination of both to robotic control problems lies in the associated computational complexity. To the best of our knowledge, this is the first work to combine general dynamic games and planning in belief space into an efficient real-time algorithm.
The main contributions of this paper are:
\begin{enumerate}
    \item A method for computing Nash equilibria for dynamic games in belief space;
    \item A linear feedback policy, similar to linear-quadratic Gaussian control (LQG), for the robot resulting from the solution, and also a predicted linear feedback policy for all other agents;
    \item Belief and control trajectory based regularization to ensure convergence;
    \item Evaluation of the proposed method in three stochastic dynamic games: racing with autonomous vehicles, active surveillance, and guiding eyes for a blind agent. 
\end{enumerate}{}

We organize the remainder of the paper as follows: Section~\ref{sec:formulation} introduces dynamic games in belief space, including a general definition of best response POMDPs and a Nash equilibrium formulation of the non-cooperative dynamic game. We give the resulting problem definition in Section~\ref{sec:problem} and, assuming beliefs can be represented in the form of Gaussian distributions, approximate the belief dynamics based on an \ac{EKF} detailed in Section~\ref{sec:ekf}.
Our method computes a locally-optimal solution to the best response POMDP problem with continuous state and action spaces and non-linear dynamics and observation models by iteratively solving for a local Nash equilibrium, outlined in Section~\ref{sec:approach}. We utilize a belief-space variant of \ac{iLQG} to compute the Nash equilibrium, Section~\ref{sec:iter_dyn_prog}, by solving for a local Nash equilibrium at each stage of the backward pass, see Section~\ref{sec:nash_equilibrium}. At each iteration, each agent's value function is approximated based on a quadratization around a nominal trajectory, and the belief dynamics are approximated with an extended Kalman filter. We describe regularization techniques in Section~\ref{sec:regularization} to ensure that the algorithm converges regardless of initial conditions.
Based on these findings, we introduce Algorithm~\ref{alg:nash_alg} in Section~\ref{sec:alg_dyn_game} describing the full belief-space Nash equilibrium computation.

We show the potential of our approach in Section~\ref{sec:case-studies} by presenting three multi-agent problems that combine our game-theoretic formulation with information-seeking behavior: active surveillance, guiding blind agents, and racing with autonomous vehicles. 

\section{Dynamic Games in Belief Space}
\label{sec:formulation}
\acrodef{MPC}[MPC]{Model Predictive Control}
\begin{table}
\caption{Main symbols and Notation}
\centering
{\renewcommand{\arraystretch}{1.2}
\begin{tabular}{l | p{0.7\columnwidth}}
$\bx, \bu, \bz$ & State, control input, and measurement\\
$\belief, \,\, \bs = [\belief^\top, \bu^\top]^\top$ & Belief, short for belief and controls\\
$Q^i, V^i$ & Action-value and value function of agent $i$ \\
$\pi^i$ & Optimal control policy of agent $i$ \\
$j_k$, $K_k$ & Feedforward and feedback gains at time $k$\\
$c_k^i(\belief_k, \bu_k)$,  $c_l^i(\belief_l)$ & Cost of agent $i$ at time $k$, and terminal cost\\
$\bx_{k+1} = f(\bx_k, \bu_k, \bm_k)$ & State transition with process noise $\bm_k$ \\
$\bz_k = h(\bx_k, \bn_k)$ & Measurement function with meas. noise $\bn_k$ \\
$\belief_{k+1} = \beta(\belief_k, \bu_k, \bz_{k+1})$ & Belief transition \\
$\belief = \bar{\belief} + \delta \belief$ & Nominal + perturbation, similar for $\bu, \bs$ \\
$c^i_{\bs,k}$, $c^i_{\bs\bs,k}$ & Gradiant and Hessian of $c^i$ evaluated at $\bar{\bs}_k$ \\
$g_{\bs,k}$, $W_{\bs,k}$ & Jacobians of $g$ and $W$ evaluated at $\bar{\bs}_k$ \\
$V_{\belief,k}$, $V_{\belief\belief,k}$ & Gradient and Hessian of value at time $k$ \\
$Q_{\bs,k}$, $Q_{\bs\bs,k}$ & Gradient and Hessian of action-value at $k$
\end{tabular}
}
\end{table}

We first define POMDPs in their most general form (following notation of \cite{thrun2005probabilistic, van2012motion}), then formulate the resulting game, derive the Nash Equilibrium, and present an iterative solution method.

We write the belief-space planning problem as a stochastic optimal control problem. 
Consider a system of $N$ agents $i \in \lbrace 1, ... , N \rbrace$, with agent $i$'s state at time $k$ denoted $\bx_k^i \in  \mathrm{R}^{n_{\bx^i}}$, measurement as $\bz_k^i \in \mathrm{R}^{n_{\bz^i}}$, and control input $\bu_k^i \in \mathrm{R}^{n_{\bu^i}}$. Here, $n_{\bx^i}$, $n_{\bz^i}$, $n_{\bu^i}$ define the dimensionality of agent $i$'s state, measurement, and control.
For brevity we refer to $\bx_k = [\bx_k^{1,\top}, \dots, \bx_k^{N,\top}]^\top  \in \mathrm{R}^{n_{\bx}} $ as the joint state, $\bz_k = [\bz_k^{1,\top}, \dots, \bz_k^{N,\top}]^\top  \in \mathrm{R}^{n_{\bz}} $ as the joint measurement,
and $\bu_k = [\bu_k^{1,\top}, \dots, \bu_k^{N,\top}]^\top  \in \mathrm{R}^{n_{\bu}} $ as the joint control, consisting of all agents. We refer to the joint dimensions as $n_{\bx} = \sum_i n_{\bx^i}$, $n_{\bz} = \sum_i n_{\bz^i}$, and $n_{\bu} = \sum_i n_{\bu^i}$. 
The notation ${\neg i}$ indicates all agents except $i$, e.g. $\bu_k^{\neg i}$ relates to the controls of all other agents except $i$.
We will refer to $\bu = [\bu_0, \bu_1, \dots, \bu_{l-1}]$ as the control trajectory until time $l$. The joint \emph{belief} $\belief(\bx_k)$ is defined as the distribution of the state $\bx_k$ given all past control inputs and sensor measurements, and consists of individual beliefs $\belief^i$. For brevity, we define $\bs = [\belief^\top, \bu^\top]^\top$. 

Following~\cite{thrun2005probabilistic, van2012motion}, we compute the belief by
\begin{equation}
    \belief(\bx_k) = p(\bx_k|\bu_0, \dots,\bu_{k-1}, \bz_1, \dots, \bz_k),
\end{equation}
from all past control inputs and sensor measurements.
The stochastic dynamics and observation model, here formulated in probabilistic notation as
\begin{equation}
    \bx_{k+1} \sim p(\bx_{k+1} | \bx_k, \bu_k),  \,\,\,\,\, \bz_k \sim p(\bz_k|\bx_k),
\end{equation}
allow us to forward propagate the belief given a control input $\bu_k$ and a measurement $\bz_{k+1}$ through Bayesian filtering:
\begin{equation}
    \belief(\bx_{k+1}) = \eta p(\bz_{k+1}|\bx_{k+1}) \int p(\bx_{k+1}|\bx_k, \bu_k)\belief(\bx_k) \mathrm{d} \bx_k.
    \label{eq:bayes_filter}
\end{equation}
In \eqref{eq:bayes_filter}, $\eta$ is a normalizer independent of $\bx_{k+1}$ and $\belief(\bx_{k+1})$ and, contains the uncertainty originating from the stochastic dynamics, the uncertain measurement and the uncertainty in the belief at the previous time step. We employ the shorthand $\belief_k$ to refer to $\belief(\bx_k)$. The stochastic \emph{belief dynamics} are defined by \eqref{eq:bayes_filter} and are written as
\begin{equation}
    \belief_{k+1} = \beta(\belief_k, \bu_k, \bz_{k+1}).
    \label{eq:belief_dynamics}
\end{equation}

The expected return of each individual agent $i$ under a control trajectory of all agents $\bu$, including its own control trajectory $\bu^i$, subject to uncertainty on the observed measurements $\bz$ over the horizon $l$ is determined by the action-value function $Q^i$, defined
\begin{equation}
        Q^i(\belief_0, \bu) = \mathop{\mathbb{E}}_{\bz}\left[c_l^i(\belief_l) + \sum_{k=0}^{l-1} c^i_k(\belief_k, \bu_k) \right]     \label{eq:best_response_pomdp}.
\end{equation}
Here $c_k^i(\cdot)$ denotes the cost at time $k$ and $c_l^i(\cdot)$ denotes the terminal cost of agent $i$.
Since there exists an action-value function for each agent, there are $N$ distinct action-value functions $Q^i$ for $i \in \lbrace 1, ... , N \rbrace$.

We will first formulate the two problems of (1) solving the general POMDP best response game, and then (2) finding the Nash equilibrium of this game.

\begin{problem}
\textbf{POMDP Best Response Game:} Given an initial belief~$\belief_0$, for agents $i \in \lbrace 1, ..., N \rbrace$, we need to solve the stochastic optimal control problem
\begin{align}
    \pi^i = & \argmin_{\bu^i} \,\, 
    Q^i(\belief_0, \bu)  \label{eq:best_response_pomdp} \,\,\,\,\, \forall i \in \lbrace 1, ..., N \rbrace  \\ &
 s.t. \, \, \belief_{k+1} = \beta(\belief_k, \bu_k, \bz_{k+1}) , 
\end{align}
for each agent by minimizing each agent's expected cost with respect to their own controls $\bu^i$, where $Q^i(\belief_0, \bu)$ is the action-value function of agent $i$.
\label{problem1}
\end{problem}
Note that all agents' optimal policies $\pi^i$ depend on the actions of all other agents because each agent $i$ minimizes their own action-value function $Q^i(\belief_0, \bu)$. The result is a non-cooperative game~\cite{basar1999dynamic} in which all agents' policies depend on the optimal policies of all other agents $\pi^i(\pi^{\neg i})$. 
Since all policies are optimized jointly and severally, the dependence of agent $i$'s policy $\pi^i$ on other agents' controls $\bu^{\neg i}$ is resolved by inserting their optimal policy $\pi^{\neg i}$. We therefore denote $\pi^i$ instead of $\pi^i(\bu^{\neg i})$.

A general solution to \eqref{eq:best_response_pomdp} can be defined recursively by the Bellman equation:
\begin{align}
    V_l^i(\belief_l) & = c^i_l(\belief_l)    , \label{eq:Bellman} \\
    Q_k^i(\belief_k, \bu_k) & = c^i_k(\belief_k, \bu_k) +  \mathop{\mathbb{E}}_{\bz_{k+1}} \left[V_{k+1}^i(\beta(\belief_k, \bu_k, \bz_{k+1})) \right], \nonumber \\
    V_k^i(\belief_k) & = \min_{\bu^i_k} Q_k^i(\belief_k, \bu_k), \nonumber \\
    \pi_k^i(\belief_k) & = \argmin_{\bu^i_k} Q_k^i(\belief_k, \bu_k), \nonumber
\end{align}
where $V_k^i(\belief_k)$ is the value function and $\pi_k^i(\belief_k)$ the optimal policy at time $k$. Note that in $\eqref{eq:Bellman}$ the cost $c^i_k(\belief_k, \bu_k)$, the reached value function $V_{k+1}^i(\beta(\belief_k, \bu_k, \bz_{k+1}))$, and therefore the action-value function $Q_k^i(\belief_k, \bu_k)$ of agent $i$ depends not only on its own action but also on all other players' actions. This interdependence is analogous to \eqref{eq:best_response_pomdp} but formulated recursively over time.

To better capture how an agent’s action-value function depends on the controls of all other actions, we can equivalently write $Q^i(\belief_0, \bu) = Q^i(\belief_0, \bu^i, \bu^{\neg{i}})$.
More precisely, the interdependence of all players optimal policies is captured in the Nash equilibrium of Problem~\ref{problem1}, defined in Problem \ref{problem2}. Problem~\ref{problem2} formulates a sufficient condition for Nash equilibria~\cite{basar1999dynamic, osborne1994course} in belief space.
\begin{problem}
$\textbf{Nash Equilibrium:}$
Find the optimal control policy $\pi = [\pi^{1,\top}, \dots, \pi^{N,\top}]^\top$ that yields a local Nash equilibrium of the POMDP Best Response Game in Problem~\ref{problem1}, such that it satisfies
\begin{align}
\label{eq:problem2}
Q^{i}(\belief_0, \bu^i, \pi^{\neg i}) \geq Q^{i}(\belief_0,\pi^i, \pi^{\neg i}), \forall i \in \{1,2,\dots,N\},
\end{align}
for all $\bu^i$ in the neighborhood of $\pi^{i}$.
\label{problem2}
\end{problem}
More intuitively, in the Nash equilibrium no player has anything to gain by changing only their own strategy.
Based on the necessary condition of Problem~\ref{problem2}, we will derive a local necessary condition for each sub-problem in the backward pass of our game-theoretic variant of belief \ac{iLQG}.

\subsection{Problem Formulation}
\label{sec:problem}

The difficulty in solving POMDPs stems from the infinite-dimensional space of all beliefs, and that in general the value function cannot be expressed in parametric form. To overcome these challenges we describe beliefs by Gaussian distributions, approximating the belief dynamics using an \ac{EKF}, and a quadratic approximation of the value function about a nominal trajectory through the belief space. We iteratively compute a local Nash equilibrium over all agents in the proximity of the nominal trajectory by solving the necessary condition ~\eqref{eq:problem2} of Problem~\ref{problem2} at each timestep during a belief-space variant of \ac{iLQG} to perform the Bellman backward recursion in \eqref{eq:Bellman}. Due to its similarity to \ac{iLQG} we benefit from linear scaling $\mathcal{O}(l)$ in the planning horizon $l$, in contrast to point-based POMDP algorithms which scale exponentially.

We are given non-linear stochastic dynamics and observation models in state-transition notation:
\begin{align}
    \bx_{k+1} &= f(\bx_k, \bu_k, \bm_k), & \bm_k \sim \mathcal{N}(0,I), \\ \bz_k & = h(\bx_k, \bn_k) & \bn_k \sim \mathcal{N}(0,I), 
\end{align}
where $\bm_k$ and $\bn_k$ are the motion and measurement noise, respectively. 
Without loss of generality, we draw both the motion and measurement noise from independent Gaussian distributions with zero mean and unit variance since the noise can be arbitrarily transformed inside these functions. Depending on the system, motion and sensing noise may be state and control dependent. 

Note that formulating the general dynamics and measurement functions jointly of all agents includes, but is not limited to, the special case of independent functions for each agent $i$ as in 
\begin{multline}
    f(\bx_k, \bu_k, \bm_k) = [f^1(\bx_k^1, \bu_k^1, \bm_k^1)^{\top}, \dots, \\ f^N(\bx_k^N, \bu_k^N, \bm_k^N)^{\top}]^\top ,
\end{multline}
\begin{multline}
    h(\bx_k, \bn_k) = [h^1(\bx_k^1, \bn_k^1)^{\top}, \dots, h^N(\bx_k^N, \bn_k^N)^{\top}]^\top.
\end{multline}

We define the Gaussian belief as $\belief_k = (\hat{\bx}_k^\top, \Sigma_k)$, 
by the mean state $\hat{\bx}_k$ and the variance $\Sigma_k$ of the Normal distribution describing the stochastic state $\bx_k \sim \mathcal{N}(\hat{\bx}_k, \Sigma_k)$.

\section{Technical Approach}
\label{sec:approach}

Before detailing the value iteration method for the Nash equilibrium solution based on a game-theoretic belief-space variant of \ac{iLQG} in Section~\ref{sec:iter_dyn_prog}, we need to derive two important components.
First, we describe the approximation of the general Bayesian filter update~\eqref{eq:belief_dynamics} by an \ac{EKF} in Section~\ref{sec:ekf} to formulate the Gaussian belief dynamics. This allows us to forward propagate Gaussian beliefs given an initial belief and a control trajectory which we utilitze in the game-theoretic variant of belief-space \ac{iLQG}. Second, we show that the necessary condition of Problem~\ref{problem2}, the Nash equilibrium, is equivalent to a local necessary condition at each timestep in the Bellman recursion in Section~\ref{sec:nash_equilibrium}.
The full algorithm is detailed in Section~\ref{sec:alg_dyn_game}.

\subsection{Bayesian Filter and Belief Dynamics}
\label{sec:ekf}
The Bayesian filter in \eqref{eq:belief_dynamics} defines the general belief dynamics of a current belief $\belief_k$ and measurement $\bz_{k+1}$. To make the belief propagation tractable we follow \cite{van2012motion} and approximate the Bayesian filter by an \ac{EKF}, suitable for non-linear Gaussian beliefs as well as non-linear dynamics and measurement models. For well-defined transition models, the \ac{EKF} is the standard for nonlinear state estimation~\cite{wan2006sigma, julier2004unscented}.
The \ac{EKF} makes a first-order approximation of $f$ with respect to the stochastic variable $\bx$,
such that for a given belief $\belief_k = (\hat{\bx}_k, \Sigma_k)$ we have the standard \ac{EKF} update equations \cite{van2012motion, thrun2005probabilistic}
\begin{align}
    \hat{\bx}_{k+1} & = f(\hat{\bx}_k, \bu_k, 0) + K_k(\bz_{k+1} - h(f(\hat{\bx}_k, \bu_k, 0), 0)), \nonumber \\
    \Sigma_{k+1} & = \Gamma_{k+1} - K_k H_k \Gamma_{k+1}, 
\end{align}
with corresponding matrices defined by
\begin{align}
    \Gamma_{k+1} & = A_k \Sigma_k A_k^T + M_k M_k^T, \\
    K_k & = \Gamma_{k+1} H_k^\top (H_k \Gamma_{k+1} H_k^\top + N_k N_k^\top)^{-1}, \nonumber \\
    A_k & = \frac{\partial f}{\partial \bx}(\hat{\bx}_k, \bu_k, 0),  \,\,\, M_k  = \frac{\partial f}{\partial \bm}(\hat{\bx}_k, \bu_k, 0), \nonumber \\
    H_k & = \frac{\partial h}{\partial \bx}(f(\hat{\bx}_k, \bu_k, 0),0),  \,\,\,
    N_k = \frac{\partial h}{\partial \bn}(f(\hat{\bx}_k, \bu_k, 0),0). \nonumber
\end{align}
The noisy measurement $\bz_k$ in the belief update makes the belief dynamics stochastic.
We define $\belief_k = [\hat{\bx}_k^\top, \text{vec}(\Sigma_k)^\top]^\top$, where $\text{vec}(\Sigma_k)$ is the matrix $\Sigma_k$ reshaped into vector form and formulate the stochastic belief dynamics
\begin{align}
    \belief_{k+1} = g(\belief_k, \bu_k) + W(\belief_k, \bu_k)\bxi_k, \,\,\, \bxi_k \sim \mathcal{N}(0,I),
\end{align}
with
\begin{align}
    g_k(\belief_k, \bu_k) &= \begin{bmatrix}f(\hat{\bx}_k, \bu_k, 0) \\ \text{vec}(\Gamma_{k+1} - K_k H_k \Gamma_{k+1}) \end{bmatrix}, \\
    W_k(\belief_k, \bu_k) &= \begin{bmatrix}\sqrt{K_k H_k \Gamma_{k+1}} \\ \mathbf{0}
    \end{bmatrix}.\label{eq:W_noise}
\end{align}
Here, $\bxi_k$ is a Gaussian with dimension $n_{\bx}$ that is applied to the stochastic part of $\belief_k$, i.e. the stochastic state variable $\bx_k$.
In this form $\bxi_k$ represents both measurement noise $\bn_k$ and motion noise $\bm_k$ mapped onto the belief transition.
The stochastic Gaussian belief dynamics allow us to propagate beliefs efficiently during the forward pass of the game-theoretic variant of belief-space \ac{iLQG}.

\subsection{Nash Equilibrium Necessary Condition}
\label{sec:nash_equilibrium}
While formulating how to propagate uncertainty for the continuous POMDP, we also need to define a tractable procedure to solve for Nash equilibria.
One common method to solve for Nash equilibria is the method of Iterated Best Response \cite{williams2018best, spica1801real}, where control policies are exchanged after each agent's separate and independent optimization iteration. In contrast, we directly integrate the necessary condition of the Nash equilibrium into the backward pass of a belief space variant of \ac{iLQG}. Specifically, we solve a quadratic game at each stage of the backward pass with a unique solution.
First, we formulate the necessary condition of Problem~\ref{problem2} as
\begin{align}
	\frac{\partial Q^{i}(\belief_0, \bu)}{\partial \bu^i} = 0, \,\,\, \forall i \in \{1,2,\dots,N\},
	\label{eq:neccessary_condition}
\end{align}
which allows us to compute local Nash equilibria by solving \eqref{eq:neccessary_condition}. 
Theorem~\ref{theorem1} states an equivalent condition for $Q_k^i(\belief_k, \bu_k)$, the value function from time $k$ to $l$, defined in the Bellman recursion \eqref{eq:Bellman}.
\begin{thm}
The necessary condition of the local Nash equilibrium \eqref{eq:neccessary_condition} is equivalent to
\begin{equation}
    \frac{\partial Q_k^i(\belief_k, \bu_k)}{\partial \bu_k^i} = 0,
    \label{eq:recur_necessary_condition}
\end{equation}
for all $i \in \{1,2,\dots,N\},$ and $k \in \{0, 1,\dots,{l-1}\}$.
\label{theorem1}
\end{thm}

\begin{proof}
Recall that the conventional POMDP formulation~\eqref{eq:best_response_pomdp} in Problem~\ref{problem1} is equivalent to the recursive Bellman equation~\eqref{eq:Bellman}.
Maximizing $Q_k^i(\belief_k, \bu_k)$ with respect to $\bu_k^i$ in \eqref{eq:Bellman} yields the corresponding necessary optimality condition $\frac{\partial Q_k^i(\belief_k, \bu_k)}{\partial \bu_k^i} = 0$, the same as~\eqref{eq:recur_necessary_condition}.
Therefore, the necessary optimality condition~\eqref{eq:neccessary_condition} of Problem~\ref{problem1} is equivalent to the recursive Bellman necessary optimality condition~\eqref{eq:recur_necessary_condition} in Theorem~\ref{theorem1}.
\end{proof}
Alternatively, we can split the action-value from time $0$ into the action-value from $k$ and the cost accumulated until $k$,
\begin{equation}
    Q^i(\belief_0, \bu) = Q^i_{k}(\belief_{k}, \bu_{k}) +  \mathop{\mathbb{E}}_{\bz}\left[\sum_{t=0}^{k-1} c^i_t(\belief_t, \bu_t) \right].
\end{equation}
Taking the derivative of both sides with respect to $\bu_{k}^i$ directly implies that $\frac{\partial Q^i_{k}(\belief_{k}, \bu_{k})}{\partial \bu_{k}^i} = \frac{\partial Q^i(\belief_0, \bu)}{\partial \bu_{k}^i}$, since the cost accumulated until $k$, the second term on the right hand side, does not depend on $\bu_{k}^i$. Intuitively, current actions cannot affect costs accumulated in the past.

Concluding, if each agent $i$ finds an optimizing policy $\pi^{i}_k$ to the Bellman recursion, all $\frac{\partial Q^i_k(\belief_k, \bu_k)}{\partial \bu_k^i}=0$ necessary conditions are fulfilled at time $k$. Note that each agents' policy $\pi^{i}(\bu^{\neg i})$ depends on the other agents' inputs $\bu^{\neg i}$, where $\neg i$ indicates all other agents. Therefore, solving the Bellman recursion simultaneously for all agents defines a static game \cite{basar1999dynamic}, but more importantly a game at each stage $k$ of the backward-pass.

In the next section, we describe our solution for integrating the Nash equilibrium necessary condition
at every time $k$ into the backward pass of a belief-space variant of \ac{iLQG}.

\subsection{Iterative Dynamic Programming}
\label{sec:iter_dyn_prog}
In this section we describe our belief-space variant of \ac{iLQG} for computing local Nash equilibria by solving the Bellman recursion defined in \eqref{eq:Bellman}. We denote the nominal belief as $\bar{\belief} = \belief  - \delta \belief$, the nominal controls $\bar{\bu} = \bu  -\delta \bu$, and $\bar{\bs} = \bs  - \delta \bs$, with $\bar{\bs} = [\bar{\belief}^\top, \bar{\bu}^\top]^\top$ and local perturbations $\delta \bu$, $\delta \belief$, $\delta \bs$.
At each iteration, the algorithm performs a backward pass and a forward pass on the current estimate of the belief $\bar{\belief} = [\bar{\belief}_0, \bar{\belief}_1, \dots, \bar{\belief}_l]$ and control trajectory $\bar{\bu} = [\bar{\bu}_0, \bar{\bu}_1, \dots, \bar{\bu}_{l-1}]$, i.e. the nominal trajectories. 
In the backward pass, the algorithm approximates the value functions for each agent as a quadratic function
\begin{equation}
    V_{k}^i (\bar{\belief}_{k} + \delta \belief_{k}) \approx  V^i_{k} + V^{i,\top}_{\belief,{k}}  \delta \belief_{k} + \frac{1}{2} \delta\belief_{k+1}^\top V^i_{\belief\belief,{k}} \delta \belief_{k}, \nonumber
\end{equation}
along the nominal trajectory, and computes a linear feedback policy $\pi^1$ for the robot and predicted linear feedback policies $\pi^{\neg 1}$ for all other agents. The value function is propagated backwards in time. In the forward pass we produce a new nominal trajectory based on the value function computed in the backward pass and apply the associated feedback policy.
This iterative process continues towards a locally optimal solution to the Nash equilibrium in belief space.  
The key idea is to maintain a quadratic approximation of $Q^i_k(\belief_k, \bu_k)$ and the value functions $V^i_k(\belief_k)$.

We first derive the quadratic form of $Q^i_k(\belief_k, \bu_k)$ in Theorem \ref{thm:quadratic_Q} by a Taylor expansion of the dynamics and costs, then find the minimizing control policy $\pi_k = [\pi^{1,\top}_k, \dots, \pi^{N,\top}_k]^\top$ by solving the static game and computing the Nash equilibrium over all agents. From this result we compute the value functions $V_k^i(\belief_k) = Q^i_k(\belief_k, \pi_k)$ and derive an update law for $V$ backwards in time.

\begin{thm}
\label{thm:quadratic_Q}
By linear expansion of the belief dynamics and quadratic expansion of the cost and value function,  $Q^i_k(\bs_k)$ is a quadratic of the form
\begin{equation}
    Q^i_k(\bar{\bs}_k + \delta \bs_k) \approx  Q^i_k + Q^{i,\top}_{\bs,k}  \delta \bs_k + \frac{1}{2} \delta\bs_k^\top Q^i_{\bs\bs,k} \delta \bs_k,
    \label{eq:quadratic}
\end{equation}
where
\begin{align}
    Q^i_k & = c^i_k + V^i_{k+1} +  \frac{1}{2}\sum_{j=1}^{n_x} W_k^{(j),\top} V^i_{\belief \belief,k+1} W_k^{(j)} \label{eq:q_vals1}, \\
    Q^i_{\bs,k} & = c^i_{\bs,k} + g_{\bs,k}^\top V^i_{\belief,k+1} + \sum_{j=1}^{n_x}  W_{\bs,k}^{(j),\top} V^i_{\belief \belief,k+1} W_{k}^{(j)}\label{eq:q_vals2} , \\
    Q^i_{\bs \bs,k} & = c^i_{\bs\bs,k} + g_{\bs,k}^\top V^i_{\belief \belief,k+1} g_{\bs,k}  + \sum_{j=1}^{n_x} W_{\bs,k}^{(j),\top} V^i_{\belief \belief, k+1} W_{\bs,k}^{(j)}. \label{eq:q_vals4}
\end{align}
\end{thm}
Similar derivations in \ac{iLQG}~\cite{todorov2005generalized}, and belief-space iLQG~\cite{van2012motion} showed an agent's action-value $Q$ function to be quadratic with respect to the agent's controls and state or belief. In contrast, we show that agent $i$'s action-value function $Q^i$ is also a quadratic with respect to the \emph{joint} state and controls which is critical to formulate the static quadratic game in the backward pass.
\begin{proof}
We start by expanding the terms of the action-value function of the Bellman recursion \eqref{eq:Bellman},
\begin{align}
\label{eq:Q_function_exp}
    Q^i_k(\belief_k, \bu_k) & = c^i_k(\belief_k, \bu_k) \\ & + \mathop{\mathbb{E}}_{\bxi_k} 
    \left[ V_{k+1}^i(g_k(\belief_k, \bu_k) + W_k(\belief_k, \bu_k)\bxi_k) \right] \nonumber,
\end{align}
to second order around the nominal control and belief $\bar{\bs}_k = [\bar{\belief}_k^\top, \bar{\bu}_k^\top]^\top$. The term $c^i_k(\belief_k, \bu_k)$ becomes
\begin{equation}
    c^i_k(\bar{\bs}_k + \delta \bs_k) \approx c^i_k +  c^{i,\top}_{\bs,k} \delta \bs_k + \frac{1}{2} \delta \bs_k^\top c^i_{\bs\bs,k} \delta \bs_k,
    \label{eq:cost_exp}
\end{equation}
with $c^i_k = c^i_k(\bar{\bs})$, where $c^i_{\bs,k}$ and $c^i_{\bs\bs,k}$ are the Jacobian and Hessian of $c^i_k$ evaluated at $\bar{\bs}_k$.
To expand the second term on the right hands side of \eqref{eq:Q_function_exp} we first expand the stochastic joint belief dynamics to
\begin{align}
    g_k(\bar{\bs}_k + \delta \bs_k) & \approx g_k + g_{\bs,k} \delta \bs_k, \label{eq:dyn_exp1}\\
    W_k^{(j)}(\bar{\bs}_k + \delta \bs_k) & \approx W_k^{(j)} + W_{\bs,k}^{(j)} \delta \bs_k , \label{eq:dyn_exp2}
\end{align}
with terms  $g_k = g_k(\bar{\bs}_k)$, $W_k^{(j)} =  W_k^{(j)}(\bar{\bs}_k)$, and $g_{\bs,k}$, $W_{\bs,k}^{(j)}$ the respective Jacobians evaluated at $\bar{\bs}_k$.
$W_k^{(j)}$ denotes the $j$-th column of matrix $W_k$.

We now formulate the second term of \eqref{eq:Q_function_exp}.
We define the value function as a quadratic around $\bar{\belief}_{k+1}$
\begin{align}
    V_{k+1}^i &(\bar{\belief}_{k+1} + \delta \belief_{k+1}) \\
    &\approx  V^i_{k+1} + V^{i,\top}_{\belief,{k+1}}  \delta \belief_{k+1} + \frac{1}{2} \delta\belief_{k+1}^\top V^i_{\belief\belief,{k+1}} \delta \belief_{k+1} \nonumber \\
    & = V^i_{k+1} + V^{i,\top}_{\belief,{k+1}}  (\belief_{k+1} - \bar{\belief}_{k+1})  \label{eq:value_exp}\\ 
    & \,\,\,\,\,\,\, + \frac{1}{2} (\belief_{k+1} - \bar{\belief}_{k+1})^\top V^i_{\belief \belief,{k+1}} (\belief_{k+1} - \bar{\belief}_{k+1}), \nonumber
\end{align}
with $\delta \belief_{k+1} = \belief_{k+1} - \bar{\belief}_{k+1}$ for convenience. Inserting the expanded dynamics \eqref{eq:dyn_exp1}, \eqref{eq:dyn_exp2} into the second term of \eqref{eq:Q_function_exp}, defined by \eqref{eq:value_exp}, and evaluating the expectation over $\bxi_k$ yields
\begin{align}
    & \mathop{\mathbb{E}}_{\bxi_k} 
    \left[ V_{k+1}^i(g_k(\bs_k) + W_k(\bs_k)\bxi_k) \right] \\
    & \approx \mathop{\mathbb{E}}_{\bxi_k} \label{eq:step1}
    \bigg[ V^i_{k+1} + V^{i,\top}_{\belief, k+1} \left(g_k(\bs_k) + W_k(\bs_k)\bxi_k - \bar{\belief}_{k+1} \right) \nonumber\\
    & \, \, \, \, \, \, \,+ \frac{1}{2}\big(g_k(\bs_k) + W_k(\bs_k)\bxi_k - \bar{\belief}_{k+1} \big)^\top V^i_{\belief \belief,k+1} \big( g_k(\bs_k) + \nonumber \\ 
    & \, \, \, \, \, \, \,W_k(\bs_k)\bxi_k - \bar{\belief}_{k+1} \big)\bigg] \\
    & = V^i_{k+1} + V^{i,\top}_{\belief,k+1} \left(g_k(\bs_k) - \bar{\belief}_{k+1} \right) \label{eq:step2}\\
    & \, \, \, \, \, \, \,+ \frac{1}{2}\big(g_k(\bs_k) - \bar{\belief}_{k+1} \big)^\top V^i_{\belief \belief, k+1} \big( g_k(\bs_k)- \bar{\belief}_{k+1} \big) \nonumber \\
    & \, \, \, \, \, \, \,+ \frac{1}{2}\text{tr}\left( W_k(\bs_k)^\top V^i_{\belief \belief, k+1} W_k(\bs_k)\right) \nonumber \\
    & = V^i_{k+1} + V^{i,\top}_{\belief, k+1} g_{\bs,k} \delta \bs_k \label{eq:resolve_expectation} + \frac{1}{2}\delta \bs_k^\top g_{\bs,k}^\top V^i_{\belief \belief,k+1} g_{\bs, k} \delta \bs_k \\
    & \, \, \, \, \, \, \,+ \frac{1}{2}\sum_{j=1}^{n_x} (W^{(j)}_k + W_{\bs, k}^{(j)} \delta \bs_k)^\top V^i_{\belief \belief, k+1} (W^{(j)}_k + W_{\bs,k}^{(j)} \delta \bs_k)^\top \nonumber.
\end{align}
Here we use the value function expansion~\eqref{eq:value_exp} in \eqref{eq:step1}, and the fact that $\bar{\belief}_{k+1} = g_k(\bar{\bs}_k)$ in \eqref{eq:step2} in the form of
\begin{equation}
    g_k(\bs_k) - \bar{\belief}_{k+1} = g_k(\bs_k) - g_k(\bar{\bs}_k) \approx g_{\bs,k} \delta \bs_k.
\end{equation}
Collecting and grouping all first and second order terms of \eqref{eq:resolve_expectation} and \eqref{eq:cost_exp} we have that the resulting $Q^i_k(\bar{\bs}_k + \delta \bs_k)$ is a quadratic with coefficients given by (\ref{eq:q_vals1} - \ref{eq:q_vals4}).
\end{proof}

For notational convenience we will drop the time index $k$ for the $Q$ matrices.
We can also recover other partial derivatives of $Q^i$ from (\ref{eq:q_vals1} - \ref{eq:q_vals4}):
\begin{align}
    Q^i_{\bs} = 
    \begin{bmatrix}
    Q^i_{\belief} \\
    Q^i_{\bu^1} \\
    \vdots \\
    Q^i_{\bu^N} \\
    \end{bmatrix}, \,
    Q^i_{\bs\bs} = 
    \begin{bmatrix}
    Q^i_{\belief \belief} &  Q^i_{\belief \bu^1} & \cdots & Q^i_{\belief \bu^N}\\
    Q^i_{\bu^1 \belief} &  Q^i_{\bu^1 \bu^1} & \cdots & Q^i_{\bu^1 \bu^N} \\
    \vdots &  \vdots & \ddots & \vdots \\
    Q^i_{\bu^N \belief} &  Q^i_{\bu^N \bu^1} & \cdots & Q^i_{\bu^N \bu^N} \\
    \end{bmatrix}.
    \label{eq:stacked_stuff}
\end{align}
With $Q^i_k(\bar{\bs}_k + \delta \bs_k)$ in quadratic form from Theorem \ref{thm:quadratic_Q}, at stage $k$ each agent $i$ solves the quadratic problem
\begin{align}
    \delta \bu^{i,*}_k & = \argmin_{\delta \bu^i_k} \,\,
    Q^{i,\top}_{\bs,k}  \delta \bs_k + \frac{1}{2} \delta\bs_k^\top Q^i_{\bs\bs,k} \delta \bs_k,
    \label{eq:little_qp}
\end{align}
yielding a quadratic game in the variables $\bu_k$.
Note that each agent's optimal $\delta \bu^{i,*}_k$ depends on all other agents' $\delta \bu^{\neg{i}}_k$ as they are contained in $\delta \bs_k$. In comparison to other related methods such as \ac{iLQR}~\cite{li2004iterative}, \ac{iLQG}~\cite{todorov2005generalized}, or \ac{DDP}~\cite{jacobson1970differential} that solve a single quadratic optimization in the backward pass, we have to solve $N$ interdependent quadratic optimizations.
Nonetheless, we obtain a unique and simple solution to the quadratic game~\cite{basar1999dynamic} by stacking the $N$ optimality conditions of each interdependent optimization. Solving the resulting system of equations amounts to solving all interdependent quadratic optimizations at once.
Theorem \ref{thm:quadratic_game_sol} presents this solution.

\begin{thm}
\label{thm:quadratic_game_sol}
The solution to the quadratic game \eqref{eq:little_qp} is
\begin{equation}
\delta \bu_k^* = -\hat{Q}_{\bu\bu}^{-1} \left( \hat{Q}_{\bu} + \hat{Q}_{\bu \belief}\delta \belief_k \right),
\label{eq:opt_pert_sol}
\end{equation}
where $\hat{Q}_{\bu\bu}$, $\hat{Q}_{\bu\belief}$, $ \hat{Q}_{\bu}$, are populated from \eqref{eq:stacked_stuff}, and defined
\begin{align}
    \hat{Q}_{\bu\bu} =
    \begin{bmatrix}
    Q^1_{\bu^1 \bu}\\
    Q^2_{\bu^2 \bu}\\
    \vdots \\
    Q^N_{\bu^N \bu}\\
    \end{bmatrix}, \,\,
    \hat{Q}_{\bu \belief} =
    \begin{bmatrix}
        Q^1_{\bu^1 \belief}\\
        Q^2_{\bu^2 \belief}\\
        \vdots \\
        Q^N_{\bu^N \belief}\\
        \end{bmatrix}, \,\
        \hat{Q}_{\bu} = 
        \begin{bmatrix}
        Q^{1}_{\bu^{1}}\\
        Q^{2}_{\bu^{2}}\\
        \vdots \\
        Q^{N}_{\bu^{N}}\\
        \end{bmatrix}. \label{eq:quadra_solution}
\end{align}
\end{thm}

\begin{proof}
By taking the derivative of the objective of \eqref{eq:little_qp} and equating it to zero, the stationarity condition of \eqref{eq:little_qp} yields
\begin{equation}
    \begin{bmatrix}
    Q^i_{\bu^i\bu^i} &  Q^i_{\bu^i \bu^{\neg i}} \\
    \end{bmatrix}
    \begin{bmatrix}
    \delta \bu^i_k\\
    \delta \bu^{\neg i}_k\\
    \end{bmatrix} + Q^i_{\bu^i\belief} \delta \belief_k +Q^i_{\bu^i} = 0.
    \label{eq:sub_qp_lin}
\end{equation}
Stacking the stationarity conditions of all $N$ agents into a single system of equations we find the joint stationarity condition for all interdependent quadratic optimizations
\begin{align}
\label{eq:stacked_stat}
    \hat{Q}_{\bu\bu} \delta \bu_k
    + \hat{Q}_{\bu\belief}
    \delta \belief_k
    + \hat{Q}_{\bu} = 0,
\end{align}
where \eqref{eq:opt_pert_sol} is the solution to this system of equations.
\end{proof}

The local necessary condition of Problem~\ref{problem2}, derived in Section~\ref{sec:nash_equilibrium} holds as shown below.
\begin{corollary}
The solution \eqref{eq:opt_pert_sol} fulfills the necessary condition of the local Nash equilibrium~\eqref{eq:neccessary_condition} at time $k$.
\end{corollary}
\begin{proof}
From \eqref{eq:sub_qp_lin}, we see that $\frac{\partial Q_k^i(\belief_k, \bu_k)}{\partial \bu_{k}^i} = 0$.
\end{proof}
We can immediately derive the linear feedback policy for all agents at planning time $k$ of the form
\begin{equation}
    \pi_k = \bar{\bu}_k + j_k + K_k \delta {\belief}_k
    \label{eq:opt_policy}
\end{equation}
with $j_k = -\hat{Q}_{\bu\bu}^{-1}\hat{Q}_{\bu}$ the feed forward term and $K_k = -\hat{Q}_{\bu\bu}^{-1} \hat{Q}_{\bu \belief}$ the feedback term.
Note that $\pi_k$ contains the optimal policy of the robot $\pi^1_k$ and also the predicted policies for all other (N-1) agents $\pi^{\neg 1}_k$. The interdependence has been resolved by solving \eqref{eq:stacked_stat}.
The predicted linear policies $\pi^{\neg 1}_k$ depend on the change in \emph{joint} belief $\delta {\belief}_k$. The predicted actions will adapt if the robot, the other agents, or the environment behave differently as expected, causing the estimated belief $\belief_k$ at future times $k$ to diverge from the predicted nominal belief $\bar{\belief}_k$. Similarly, the robot's linear policy $\pi^1_k$ will allow it to adapt if other agents deviate from predicted behavior. In contrast, this flexibility would be impossible with a static optimal control trajectory instead of a policy.

We now formulate the backward equations to propagate the value functions $V^i$ backwards, hence defining the backward pass.
\begin{corollary}
\label{thm:backwards_diff_V}
The discrete backward differential equations of the value functions $V^i$ are
\begin{align}
    V^i_{k} & = Q^i + Q^{i,\top}_{\bu} j_k + \frac{1}{2} j_k^\top Q^{i}_{\bu \bu} j_k , \label{eq:v_update0}\\
    V^i_{\belief, k} & = Q^i_\belief + K_k^\top Q^i_{\bu\bu} j_k
   + K_k^\top Q^i_{\bu} + Q^{i,\top}_{\bu \belief} j_k ,  \label{eq:v_update1}\\
    V^i_{\belief \belief,k} &= Q^i_{\belief \belief} + K_k^\top Q^i_{\bu \bu} K_k + K_k^\top Q^i_{\bu \belief} + Q^{i,\top}_{\bu \belief} K_k,\label{eq:v_update2}
\end{align}
with terminal constraints 
\begin{align}
V^i_{l} = c_l^i(\bar{\belief}_l), \,\,
    V^i_{\belief, l} = \eval{\frac{\partial c_l^i(\belief)}{\partial \belief} }_{\belief=\bar{\belief}_l} , \,\,  V^i_{\belief \belief, l} = \eval{\frac{\partial^2 c_l^i(\belief)}{\partial \belief^2} }_{\belief=\bar{\belief}_l}. \label{eq:term_cond}
\end{align}
\end{corollary}
\begin{proof}
Substituting the solution \eqref{eq:opt_pert_sol} and \eqref{eq:opt_policy} back into the quadratic \eqref{eq:quadratic} yields the value function $V_k^i (\bar{\belief}_k + \delta \belief_k)$.
\begin{align}
    &V_k^i (\bar{\belief}_k + \delta \belief_k) = Q^i_k(\bar{\belief}_k + \delta \belief_k, \pi_k) \nonumber \\
    & = Q^i 
    + Q^{i,\top}_{\bu}(j_k + K_k \delta \belief_k) 
    + Q^{i,\top}_{\belief} \delta \belief_k \nonumber \\
    &+ \frac{1}{2} (j_k + K_k \delta \belief_k)^{\top} Q^i_{\bu \bu} (j_k + K_k \delta \belief_k) 
    + \frac{1}{2}  \delta \belief_k^\top Q^i_{\belief \belief} \delta \belief_k \nonumber\\
    &+ \frac{1}{2}  (j_k + K_k \delta \belief_k)^{\top} Q^i_{\bu\belief} \delta \belief_k  
    + \frac{1}{2}  \delta \belief_k^{\top} Q^i_{\belief \bu} (j_k + K_k \delta \belief_k). \nonumber
\end{align}
Collecting first and second order terms in $\delta \belief_k$ gives the Equations~(\ref{eq:v_update0}-\ref{eq:v_update2}) in the form of \eqref{eq:value_exp}. The terminal constraints \eqref{eq:term_cond} result from a Taylor expansion of the final cost $c^i_l$ around the final nominal belief $\bar{\belief}_l$.
\end{proof}
Based on results of Theorem~\ref{thm:quadratic_game_sol} and Corollary~\ref{thm:backwards_diff_V} we can propagate the quadratic value functions backwards in time starting from the terminal constraints at time $l$.

\subsection{Regularization}
\label{sec:regularization}

With any Newton-like method, care must be taken when the Hessian $\hat{Q}_{\bu\bu}$ is not positive-definite or when the minimum is not close and the quadratic model inaccurate. To ensure that the algorithm converges regardless of initial conditions, we implement a Levenberg-Marquardt style regularization~\cite{levenberg1944method}.
\subsubsection{Control Regularization}
The control regularization is achieved by adding a diagonal term of magnitude $\mu_\bu$ to the diagonal of $\hat{Q}_{\bu\bu}$, yielding
\begin{equation}
\label{eq:control_regu}
\tilde{Q}^i_{\bu \bu} = \hat{Q}^i_{\bu \bu} + \mu_\bu I.
\end{equation}
This simple Levenberg-Marquardt style modification results in adding a quadratic cost around the current control sequence, which forces the new optimal control inputs computed by the backward pass to stay closer to the previous iteration.
\subsubsection{Belief Regularization}
The drawback of the control-based regularization scheme is that even small control perturbations can cause large deviations in the state trajectory potentially inhibiting convergence. To ensure that the updated belief trajectory does not deviate too far from the previous iteration, we introduce a scheme that penalizes deviations from beliefs rather than controls with parameter $\mu_\belief$:
\begin{align}
\label{eq:state_regu}
    \tilde{Q}^i_{\bs \bs,k} & = c^i_{\bs\bs,k} + g_{\bs,k}^T \left( V^i_{\belief \belief,k+1} + \mu_\belief I \right)g_{\bs,k} \\
    & \, \, \, \, \, \, \, + \sum_{i=1}^n W_{\bs,k}^{(j),T}\left( V^i_{\belief \belief,k+1} + \mu_\belief I \right)W_{\bs,k}^{(j)}. \nonumber
\end{align}
The belief-based regularization results in placing a quadratic belief-cost around the previous belief trajectory, similarly to \cite{tassa2012synthesis} where a state-based regularization was employed. In contrast to the standard control-based regularization, the feedback gains $K_k$ do not go to zero as $\mu_{\belief} \to \infty$, but rather force the new trajectory closer to the old one. In practice, we find this to improve the robustness of convergence.

\subsection{Algorithm for Dynamic Game Belief Space Planning}
\label{sec:alg_dyn_game}
We summarize our findings of solving Nash equilibria of dynamic games in belief space in Algorithm~\ref{alg:nash_alg}.
Theorem~\ref{thm:quadratic_Q} lays the foundation for the quadratic game solved in the backward pass of Algorithm~\ref{alg:nash_alg}. The solution to the quadratic game presented in Theorem~\ref{thm:quadratic_game_sol} yields a linear feedback policy $\pi_k$ for all agents. We propagate the value function in the backward pass according to Corollary~\ref{thm:backwards_diff_V} starting with the terminal conditions from the terminal cost.

Algorithm~\ref{alg:nash_alg} starts from the current belief estimate $\belief_0$, in our experiments provided from an \ac{EKF}, and an initial control trajectory guess. We found initializing controls to all zeros to work well in practice.
We update the nominal control and belief trajectories in the forward pass based on rolling out the belief dynamics model and applying the updated feedback policy $\pi_k$. If all agents' action-value functions improved, we accept the updated nominal belief and control trajectories and reduce regularization. Otherwise, we reject the trajectories and increase regularization.
The iteration of backward and forward pass continues until each agents' action value function $Q^i$ has converged and changes less than a specified threshold $\epsilon$.
\begin{algorithm}[H]
\caption{Nash Equ. of Dynamic Games in Belief Space}
    \hspace*{\algorithmicindent} \textbf{Input:} Initial belief $\belief_0$, control $\bar{\bu}$, models $c_k^i$, $c_l^i$, $f$, $h$\\
    \hspace*{\algorithmicindent} \textbf{Output:} Predicted trajectories $\bar{\belief}$, $\bar{\bu}$, feedback law $\pi$  
\begin{algorithmic}[1]
\State $\bar{\belief}$ $\gets$ Propagate $\belief_0$ with $g$ and $\bar{\bu}$
\While{$|Q^i(\bar{\belief}_\text{new}, \bar{\bu}_\text{new}) - Q^i(\bar{\belief}, \bar{\bu})| > \epsilon$}
\State \textbf{Backward pass}:
    \State $V^i_{\belief,l}$, $V^i_{\belief\belief,l}$ $\gets$ From terminal boundary conditions \eqref{eq:term_cond}
\For{$k$ from $l-1$ to $0$} 
\State $\pi_k^i$, $j^i_k$, $K^i_k$ $\gets$ Solve quadratic game \eqref{eq:opt_policy}
\State $V^i_{\belief,k}$,$V^i_{\belief \belief, k}$ $\gets$ Propagate value function (\ref{eq:v_update1}, \ref{eq:v_update2})
\EndFor
    \State \textbf{Forward pass}:
    \State $\bar{\belief}_\text{new}$, $\bar{\bu}_\text{new}$  $\gets$ Propagate $\belief_0$ with $g$ and $\pi$
    \If {$Q^i(\bar{\belief}_\text{new}, \bar{\bu}_\text{new}) \leq Q^i(\bar{\belief}, \bar{\bu})$}
    \State $\bar{\belief}, \bar{\bu}\gets \bar{\belief}_\text{new}, \bar{\bu}_\text{new}$,
    \State lower regularization (\ref{eq:control_regu}, \ref{eq:state_regu})
\Else \,\,\, increase regularization
\EndIf
\EndWhile 
\end{algorithmic}
\label{alg:nash_alg}
\end{algorithm}
The algorithm yields a linear feedback policy $\pi^1$ and a predicted belief trajectory $\belief^1$ of the robot. It also gives predicted feedback policies $\pi^{\neg 1}$ and predicted belief trajectories $\belief^{\neg 1}$ for all other agents over the full time horizon.

\subsection{Runtime Analysis}
The dominant runtime complexity in a single backward step is $\mathcal{O}(N^7 n_\bx^{i,6})$.
A full iteration of Algorithm~\ref{alg:nash_alg} solves $l$ quadratic games leading the final runtime complexity to $\mathcal{O}(l N^7 n_\bx^{i,6})$. Scaling linearly in the planning horizon $l$ enables real-time deployment whereas other POMDP algorithms scale exponentially, even without taking any game dynamics into account. The following provides a brief summary of our runtime analysis.

We analyze the runtime by first recalling that the dimension of the joint state is $\mathcal{O}(n_\bx)$, and assume for the sake of analysis that the agents' state dimensions are equal such that $\mathcal{O}(n_\bx) = \mathcal{O}(N n_\bx^i)$. To simplify the analysis, we also assume the joint input $(n_\bu)$ and the joint measurement dimensions $(n_\bz)$ to be $\mathcal{O}(n_\bx)$. The covariance matrix of the joint state contains $n_\bx^2/2$ unique elements. Since the joint belief $\belief$ contains the covariance of the state in addition to the state itself, it entails $\mathcal{O}(n_\bx^2)$ elements. 

Now consider the matrix multiplicative terms in the iterative dynamic programming procedure. A computational bottleneck occurs when updating the action-value function $Q^i_{\bs \bs}$ in \eqref{eq:q_vals4}. Evaluating the product $g_{\bs}^\top V^i_{\belief \belief} g_{\bs}$ requires the multiplication of matrices with dimensions $\mathcal{O}(n_\bx^2) \times \mathcal{O}(n_\bx^2)$, resulting in $\mathcal{O}(n_\bx^6) = \mathcal{O}(N^6 n_\bx^{i,6})$ complexity. This operation has to be completed for each of the $N$ agents such that the complexity increases to $\mathcal{O}(N^7 n_\bx^{i,6})$. The term $W_{\bs}^{(j),\top} V^i_{\belief \belief} W_{\bs}^{(j)}$ in \eqref{eq:q_vals4} must be computed $n_\bx$ times, but can be evaluated in $\mathcal{O}(n_\bx^5)$, since $W$ only contains $n_\bx$ non-zero elements. See \eqref{eq:W_noise} for the definition of $W$.
We solve the quadratic game at each stage by finding the inverse of the $n_\bu \times n_\bu$ matrix $\hat{Q}_{\bu\bu}$ in \eqref{eq:quadra_solution}, which has complexity $\mathcal{O}(n_\bx^3)$. Therefore, $\mathcal{O}(N^7 n_\bx^{i,6})$ remains the dominant runtime complexity.

Next, we investigate the complexity of evaluating derivatives, Hessians, and Jacobians. 
The cost Hessian $c^i_{\bs\bs}$ only contains $\mathcal{O}(n_\bx^4)$ elements.
Automatic differentiation through source code transformation yields $\mathcal{O}(1)$ complexity for each element, such that the cost Hessian term has no significant impact on the overall runtime complexity. The \ac{EKF} belief dynamics can be evaluated in $\mathcal{O}(n_\bx^3)$, such that linearizing the belief dynamics to obtain $W_\bs$ and $g_\bs$, both with $\mathcal{O}(n_\bx^2)$ entries, results in $\mathcal{O}(n_\bx^5)$.

Thus, we find the dominant runtime complexity in a single backward step is $\mathcal{O}(N^7 n_\bx^{i,6})$, and a final runtime complexity of $\mathcal{O}(l N^7 n_\bx^{i,6})$ in a full iteration of Algorithm~\ref{alg:nash_alg}.

\section{Case Studies}
\label{sec:case-studies}

We demonstrate the performance and flexibility of our algorithm in three case studies that combine the information-seeking behavior with our game-theoretic formulation. These case studies examine how the agents interact in the game, gain information, and use the information gain to improve their control policies. We choose these illustrative examples due to their variations in agent interactions and demonstration of broader capabilities. Each of the case studies employs a different dynamics and observation model as well as distinct objectives for the agents. We find the Nash equilibrium to each of these games through Algorithm~\ref{alg:nash_alg}.

\subsection{Active Surveillance}
In this case study, Agents 1 and 2 are in an environment with variable lighting conditions. Agent 1 is tasked with observing Agent 2, but the quality of the observations depends on the available lighting at the location in the environment. Agent 2 has no goal but is assigned the objective of maintaining a constant velocity while avoiding Agent 1. In the provided examples, the agents do not directly exchange information. Instead, they perceive themselves and the other agent only through observations. At the time of initialization, neither agent has perfect information of the other but only a noisy state estimate defining the initial belief $\belief_0$. Using our approach, we show that Agent 1 can successfully herd Agent 2 into the lighted region to achieve its surveillance objective, which would not be possible without incorporating the belief-space planning into the dynamic game. Figures \ref{fig:active_surveillance1} and \ref{fig:active_surveillance2} show the planned trajectories in two environments. 
Our case study goes beyond the commonly studied multi-robot herding problem~\cite{Lien2005, pierson2015bio, Pierson2018tro, Strombom2014} which has the goal of herding agents into a specified location. In contrast, our goal is to reduce uncertainty in the final state of another agent, which happens to coincide with pushing the other agent into the light. Algorithms commonly applied to the herding problem are not applicable here as they do not reason about the uncertainty of other agents.

The state of both car-like robots $\bx^{(i)} = [x^{(i)}, y^{(i)}, \theta^{(i)}, v^{(i)}]$ consists of their position $(x,y)$, orientation $\theta$, and speed $v$. The control inputs $\bu^{(i)}= [u^{(i)}_{\text{acc},k},u^{(i)}_{\text{steer},k}]$ are acceleration $u_{\text{acc},k}$ and steering wheel angle $u_{\text{steer},k}$. 
The deterministic continuous dynamics of both agents are given by
\begin{equation}
    \dot{\bx}^{(i)}_k = \big[v^{(i)}_k \cos\theta^{(i)}_k,~ v^{(i)}_k \sin\theta^{(i)}_k, u^{(i)}_{\text{acc},k}, \frac{ v_k^{(i)}}{L \tan(u^{(i)}_{\text{steer},k})}  \big]^\top , \nonumber
\end{equation}
where L is the length of the robots. The discrete time dynamics are defined by
\begin{equation}
    \bx_{k+1} = f(\bx_k, \bu_k, \bm_k) = \bx_{k} + \dot{\bx}_k\tau + M(\bu_k) \cdot \bm_k, \nonumber
\end{equation}
with timestep $\tau$. $M(\bu_k)$ scales the motion noise $\bm_k$ proportional to the control input $\bu_k$, such that uncertainty increases if excessive controls are executed. We encode the agent's objective and goals in this game by defining the current and terminal costs for Agent~1 and Agent~2 as
\begin{align}
    c_k^{(1)}(\belief_k, \bu_k) & = \bu_k^{(1),\top} R \bu^{(1)}_k, \nonumber \\
    c_l^{(1)}(\belief_l) & = \text{det}(\Sigma^{(2)}_{x,y,l}), \nonumber \\
    c_k^{(2)}(\belief_k, \bu_k) & = {\bu_k^{(2),\top}} R \bu^{(2)}_k + a_1 (v_k^{(2)} - v^{(2)}_{k,\text{des}})^2 + a_2 c_\text{coll}(\bx_k),\nonumber\\
    c_l^{(2)}(\belief_l) & = a_1(v_l^{(2)} - v^{(2)}_{l,\text{des}})^2 + a_2 c_\text{coll}(\bx_l). \nonumber
\end{align}
Agent 1's overall objective is to lower the uncertainty about the position of Agent 2 at the end of the planning horizon, encoded by $c_l^{(1)}(\belief_l)$. The term $\text{det}(\Sigma^{(2)}_{x,y})$ is equivalent to the area of the $1\sigma$-threshold ellipse of Agent 2 and representative of the location uncertainty of Agent 2 at the end of the planning horizon. Note that both agents penalize control effort by $\bu_k^{(i),\top} R \bu^{(i)}_k$, and Agent 2 has additional objectives for maintaining a desired velocity $v_\text{des}$ and avoiding collisions via an exponential barrier $c_\text{coll}(\bx_k) = \exp(-d(\bx_k))$. Here $d(\bx_k)$ is the expected euclidean distance until collision between the two agents, taking their outline into account. 

We restrict the robots' sensing abilities to only include noisy position measurements. The observation model varies across the environment based on the available light at a particular location,
\begin{equation}
    \bz^{(i)}_k = h(\bx^{(i)}_k, \bn^{(i)}_k) = [x^{(i)}_k, y^{(i)}_k]^T + N(\bx_k^{(i)}) \cdot \bn^{(i)}_k, \nonumber
\end{equation}
where the matrix $N(\bx_k^{(i)})$ scales the measurement noise based on the current position $(x,y)$ in the map.
We show the nominal trajectories and the associated beliefs of the solution computed using Algorithm \ref{alg:nash_alg} in Figures \ref{fig:active_surveillance1} and \ref{fig:active_surveillance2}. In both cases Agent 1 (blue) is able to force Agent~2 into the light to successfully reduce uncertainty. The emergent behavior would not have been possible without belief-space planning, reasoning about another agent's uncertainty, and without the dynamic game, estimating how the own actions influence another agent's actions.
We show the resulting behavior without belief-space planning and without any reasoning about Agent~2's uncertainty in the inset of Figure~\ref{fig:active_surveillance2}.

\begin{figure}
    \centering
    \includegraphics[width=1.0\columnwidth]{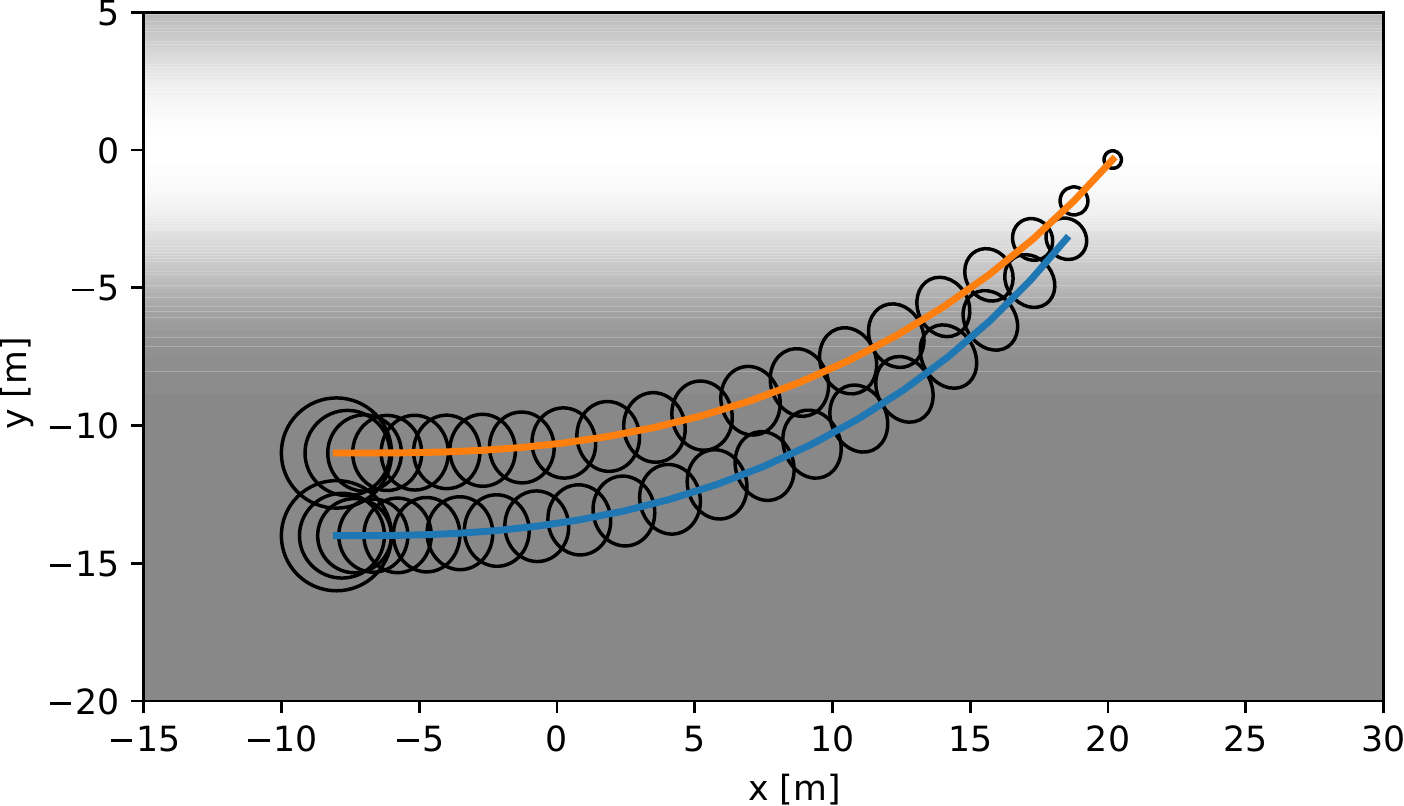}
    \caption{Agent 1 (blue) is pushing Agent 2 (orange) into the light to reduce the uncertainty over Agent 2 at the end of the planning horizon. Uncertainties are visualized by covariance ellipses. Both agents are initialized with positive velocity in the x-direction.}
    \label{fig:active_surveillance1}
\end{figure}

\begin{figure}
    \centering
    \includegraphics[width=1.0\columnwidth]{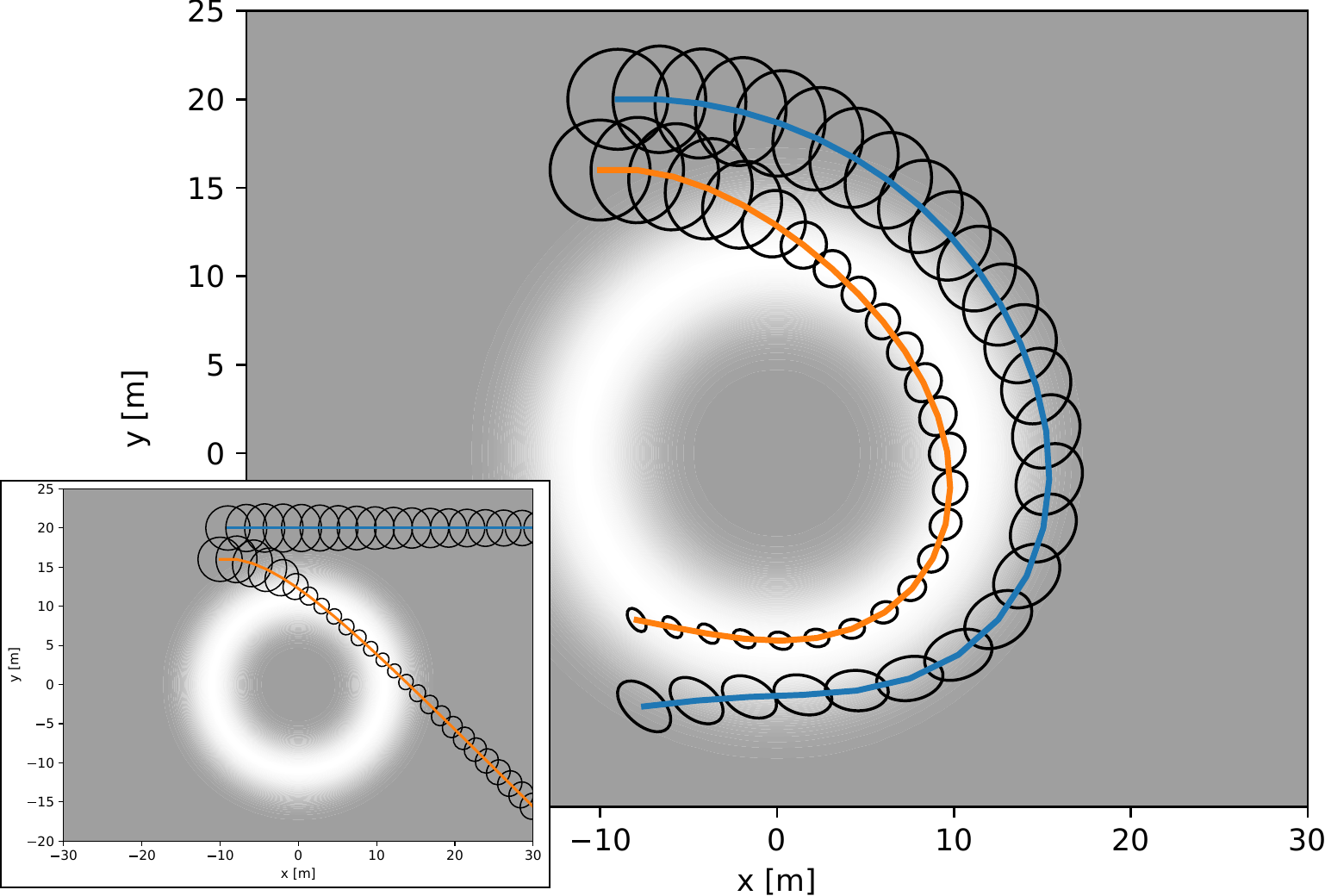}
    \caption{By nudging Agent 2 onto the circular light source Agent 1 is able to reduce the uncertainty over Agent 1's state at the end of the planning horizon. The lower left inset shows the same scenario without any information gain. As a result, Agent 1 has no incentive to manipulate Agent 2's behavior since there is no way to influence its uncertainty. Both agents start with positive velocity in the x-direction.}
     \label{fig:active_surveillance2}
\end{figure}

\begin{figure}
    \centering
    \includegraphics[width=1.0\columnwidth]{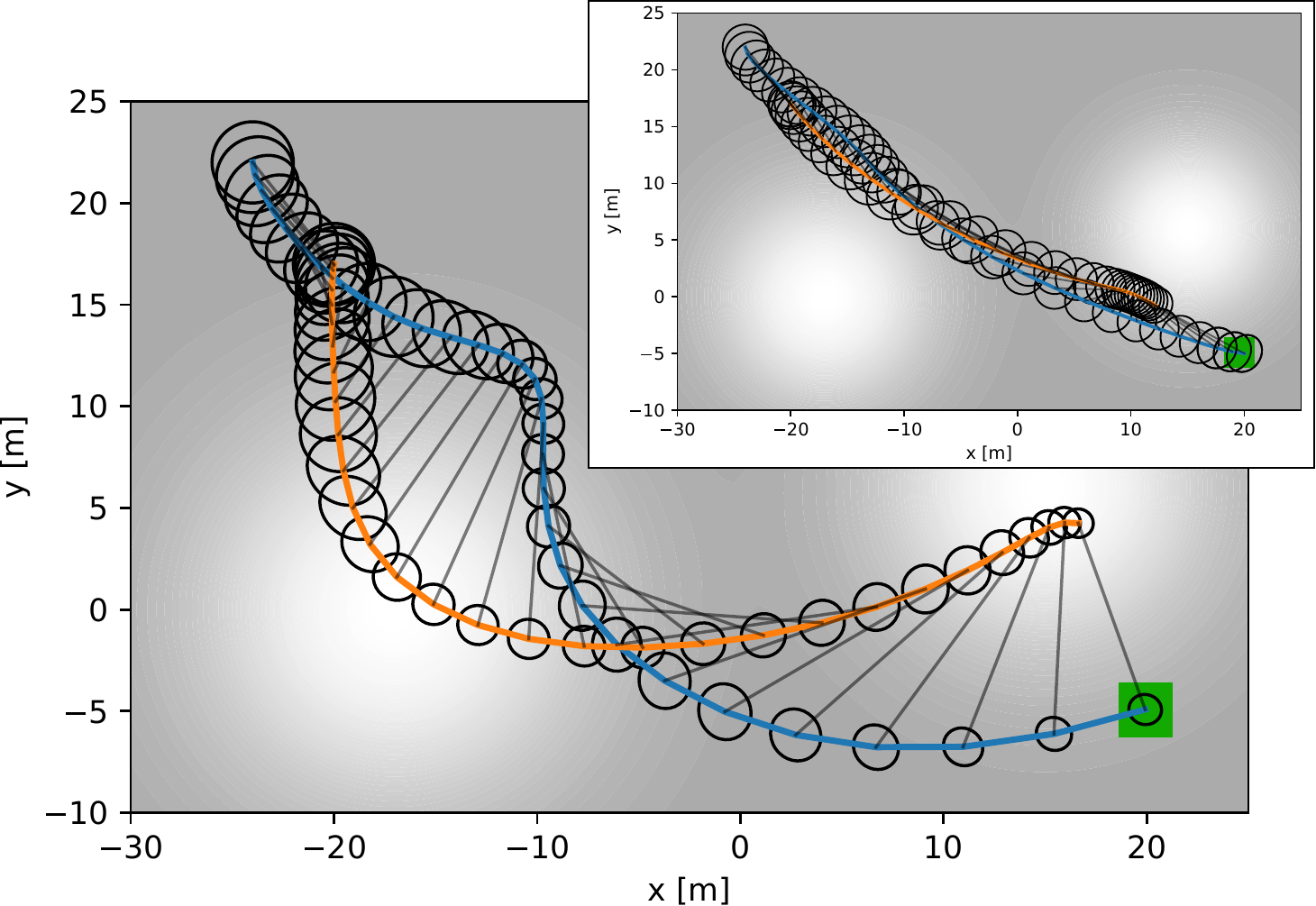}
    \caption{Guide dog (orange) with leash (black line) guides the blind agent (blue) towards the goal location (green). While doing so it passes by both light sources to reduce the uncertainty of the blind person's position at the goal location. The top right inset shows the case where the guide dog is indifferent about the blind person's uncertainty.}
    \label{fig:guide dog}
\end{figure}

\subsection{Guide Dog for Blind Agent}
In this scenario, Agent 2 guides Agent 1 towards a goal location while choosing a path that reduces the uncertainty in Agent 1's position. The game is won if Agent 1 knows it reaches the goal location with a low uncertainty about its state. However, Agent 1 does not have the ability to navigate itself. We refer to Agent 1 as the ``blind'' agent. Following this analogy, Agent 2 acts as the ``guide dog'' for the blind agent. The guide dog can gather information about its own state and the blind agent's state by passing through light sources in the environment which reduces uncertainty. The agents are tethered together, which we model with spring dynamics and refer to the tether as the ``leash.'' If the guide takes the blind agent on the direct path to the goal, the guide would not have sufficient information to know it brought the blind agent  to the goal location. Under our approach, the guide dog detours to key areas to reduce the blind agent's uncertainty. We use the analogy of a guide dog leading a blind agent to create an intuitive visual for the reader, however, this system is relevant to many other robotic applications.

We model the system dynamics as two masses on a surface with friction connected by a spring tether. The states of the blind agent and the guide dog are $\bx^{(i)} = [\mathbf{r}^{(i)}, \mathbf{v}^{(i)}]$ the 2D position $\mathbf{r}$ and velocities $\mathbf{v}$. The inputs $\bu^{(i)} = F^{(i)}$ are their respective force vectors. The blind agent and guide dog have masses $c_\text{mass,h}$ and $c_\text{mass,d}$ respectively and are bound to friction coefficients $c_\text{fric,h}$ and $c_\text{fric,d}$. The accelerations are
\begin{align}
         \mathbf{a}^{(1)}  &= 1/c_\text{mass,h}(\bu^{(1)} - f_\text{spring}(\Delta \mathbf{r}) - c_\text{fric,h}\mathbf{v}^{(1)}), \nonumber \\
         \mathbf{a}^{(2)}  &= 1/c_\text{mass,d}(\bu^{(2)} + f_\text{spring}(\Delta \mathbf{r}) - c_\text{fric,d}\mathbf{v}^{(2)}), \nonumber
\end{align}
and influenced by the spring force 
\begin{equation}
    f_\text{spring}(\Delta {r}) = \frac{\Delta \mathbf{r} }{||\Delta \mathbf{r}|| }  c_\text{spring} \max(||\Delta \mathbf{r}|| - c_\text{leash}, 0), \nonumber
\end{equation}
which is dependent on the distance vector $\Delta(\mathbf{r}) = [\mathbf{r}^{(1)}-\mathbf{r}^{(2)}]$.
The dog's leash is flexible with spring constant $c_\text{spring}$ and has length $c_\text{leash}$, such that it only generates a spring force if extended beyond $c_\text{leash}$ and is otherwise slack. The deterministic continuous dynamics are $\dot{\bx}^{(i)} = [\mathbf{v}^{(i),\top}, \mathbf{a}^{(i),\top}]^\top$, and the discrete time dynamics 
\begin{equation}
    f(\bx_k, \bu_k, \bm_k) = \bx_{k} + \dot{\bx}_k\tau + M(\bu_k) \cdot \bm_k, \nonumber
\end{equation}
for timestep $\tau$ and where $M(\bu_k)$ scales the motion noise proportional to the inputs $\bu_k$.

We use the cost functions to encode the behaviors and objects of each agent. Similar to the previous case study, minimizing $\text{det}(\Sigma^{(1)}_{\mathbf{r},l})$ reduces the uncertainty at the end of the planning horizon. We define
\begin{align}
    c_k^{(1)} (\belief_k, \bu_k)& = \bu_k^{(1),\top} R \bu^{(1)}_k + c_\text{acc,h}\mathbf{a}_k^{(1),\top}\mathbf{a}_k^{(1)},  \nonumber \\
    c_l^{(1)}(\belief_l) & =  0, \nonumber \\
    c_k^{(2)}(\belief_k, \bu_k) & = \bu_k^{(2),\top} R \bu^{(2)}_k, \nonumber \\
    c_l^{(2)}(\belief_l) & = \text{det}(\Sigma^{(1)}_{\mathbf{r},l}) + ||\mathbf{r}_l^{(1)} - \mathbf{r}_\text{goal}||^2. \nonumber
\end{align}
Here, the term $||\mathbf{r}_l^{(1)} - \mathbf{r}_\text{goal}||^2$ drives the guide dog to relocate the blind agent to the goal. We reduce the control efforts of each agent by $\bu_k^{(i),\top} R \bu^{(i)}_k $, and the blind agent has the additional objective of reducing accelerations with  $c_\text{acc,h}\mathbf{a}_k^{(1),\top}\mathbf{a}_k^{(1)}$. We use a noisy observation model
\begin{equation}
    \bz^{(i)}_k = h(\bx^{(i)}_k, \bn^{(i)}_k) = \bx^{(i)}_k + N(\bx_k^{(i)}) \cdot \bn^{(i)}_k,
\end{equation}
where the matrix $N(\bx_k^{(i)})$ scales the measurement noise based on the environment shown in Figure~\ref{fig:guide dog}.

The resulting behavior is shown in Figure~\ref{fig:guide dog}: The guide dog (orange) guides the blind agent (blue) from its initial position to the blind person's goal location (green) while reducing the uncertainty of the blind agent's final state by planning a slight detour through the light sources instead of directly towards the goal location. The guide does so while also taking the complex interaction originating from the blind person's forces on the tether into account.
The inset on Figure \ref{fig:guide dog} is the path taken by the dog with no optimization over the blind agent's uncertainty. While it takes a direct path to the goal, the final uncertainty of the blind agent is large. 

\subsection{Autonomous Racing}
\label{sec:autonomous_racing}
 
Finally, we demonstrate our approach in competitive racing, a common problem in dynamic games. By incorporating belief-space planning into the dynamic game formulation, we show a significant increase in racing performance. This allows the agents to reduce uncertainty and decrease chance constraints. Thus, maneuvers like overtaking on tight road segments become possible. 

In all racing runs each agent maintains a separate instance of Algorithm~\ref{alg:nash_alg}. This means that each agent separately computes their own optimal control actions, the predictions of other respective agents, and their own Nash equilibrium. No other additional information, such as state estimates, beliefs, policies, or initializations are shared among agents.
Since each agent executes a separate instance of Algorithm~\ref{alg:nash_alg}, Assumption~\ref{assum:FirstOrder} may not be accurate, i.e. the belief computed by agent $j$ over agent $i$ may only inaccurately resemble the belief of agent $i$ over itself. Nonetheless, we will show that despite a first-order belief assumption, the presented approach yields superior performance to all other baselines.

Each agent's state $\bx^{(i)} = [x^{(i)}, y^{(i)}, \theta^{(i)}, v^{(i)}]$ and controls $\bu^{(i)}= [u^{(i)}_{\text{acc},k},u^{(i)}_{\text{steer},k}]$ are the same as in the active surveillance experiment but the different deterministic continuous dynamics are of the from
\begin{multline}
        \dot{\bx}^{(i)}_k = \big[v^{(i)}_k \cos(\theta^{(i)}_k), v^{(i)}_k \sin(\theta^{(i)}_k),\\ u^{(i)}_{\text{acc},k} - c_\text{drag,i}v_k^{(i)} - c_\text{slip,i}(\dot{\theta}^{(i)})^2,\dot{\theta}^{(i)}\big]^\top , \nonumber
\end{multline}
with yaw rate $\dot{\theta}^{(i)} = v_k^{(i)}/L \tan(u^{(i)}_{\text{steer},k}) $, and drag-  $c_\text{drag,i}$ and slip coefficient $c_\text{slip,i}$.
The stochastic discrete time dynamics,
\begin{equation}
    \bx_{k+1} = f(\bx_k, \bu_k, \bm_k) = \bx_{k} + \dot{\bx}_k\tau + M(\belief_k, \bu_k) \cdot \bm_k, \nonumber
\end{equation}
are subject to noise scaled by $M(\belief_k,\bu_k)$ proportional to the control input $\bu_k$ as well as the squared yaw rate $(\dot{\theta}^{(i)})^2$ of each agent $i$ separately. The observation model 
\begin{equation}
    \bz^{(i)}_k = h(\bx^{(i)}_k, \bn^{(i)}_k) = \bx^{(i)}_k + N(\bx_k^{(i)})\cdot \bn^{(i)}_k, \nonumber
\end{equation}
is subject to noise scaled by $N(\bx_k^{(i)})$, depending on the position on the race track map. As shown in Figure \ref{fig:race_big}, we indicate zones of low measurement noise as red. It may be beneficial for agents to plan to drive through these low measurement noise regions to increase information gain and to reduce uncertainty.

\begin{figure}
\centering
        \includegraphics[width=0.95\columnwidth]{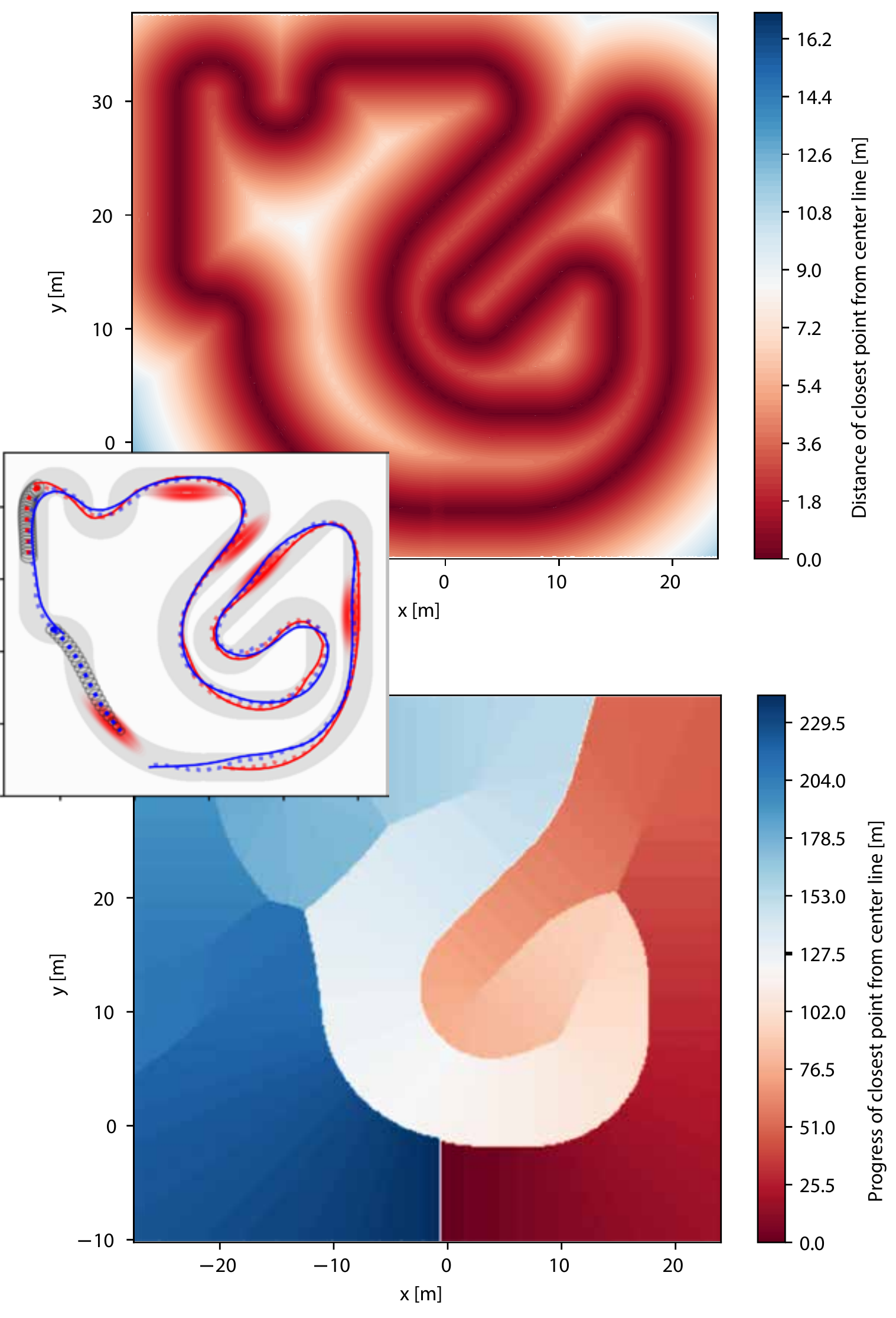}
    \caption{\textbf{Top: (Distance Transform)} Map of the distances to the closest point on the center line $d(\mathbf{p})$ of the race track shown in Figure~\ref{fig:race_big}. \textbf{Bottom: (Progress Transform)} Map of the progress $r(\mathbf{p})$ along the race track of the closest point on the center line.}
    \label{fig:dist_prog_race}
\end{figure}

Each agent's goal is to maximize progress along the race track while staying on the track and not colliding with other agents. We define the progress along the track for any point $\mathbf{p}=(x,y)$ as the arc-length progress $r(\mathbf{p})$ of the closest point on the centerline. Likewise, we define $d(\mathbf{p})$ as the distance of the closest point on the track to $\mathbf{p}$. We visualize both the distance transform as well as the progress transform of the race track shown in Figure~\ref{fig:race_big} and Figure~\ref{fig:dist_prog_race}. For competitive racing, each agent tries to maximize the relative progress over other agents $r(\mathbf{p}^{(i)}) - r(\mathbf{p}^{(\neg i)})$. Consequently, agents will engage in competitive blocking and cutting behavior. We design the current and terminal costs of each agent as
\begin{align}
    c_k^{(i)}(\belief_k, \bu_k) & = \bu_k^{(i),\top} R \bu^{(i)}_k  \nonumber
    + c^{(i)}_\text{track}(\belief_k)  
    + c^{(i)}_\text{coll}(\belief_k), \\
    c_l^{(i)}(\belief_l) & =  -r(p^{(i)}_l) + r(p^{(\neg i)}_l), \nonumber
\end{align}
penalizing control effort by $R$, while $c^{(i)}_\text{track}(\belief_k)$ and $c^{(i)}_\text{coll}(\belief_k)$ keep the agent on the track and out of collision. We achieve this by finding the upper bound of the $2\sigma$ positional uncertainty $\Sigma_{x,y}^{(i)}$ as $\alpha = 2\sqrt{\max(\text{eig}(\Sigma_{x,y}^{(i))}))}$. We can then formulate a chance collision constraint with other agents (limiting $||\mathbf{p}^{(i)}-\mathbf{p}^{(j)}||$) and the boundary of the race track (limiting $d(\mathbf{p})$) by restricting positions in the $\alpha$ vicinity.
Finally, to arrive at $c^{(i)}_\text{track}(\belief_k)$ and $c^{(i)}_\text{coll}(\belief_k)$ we convert the constraints to soft constraints, penalizing constraint violation exponentially strong, as suggested in \cite{van2012motion}. Additionally, we also limit control inputs $\bu_k$ by soft constraints.

\begin{figure}
    \centering
        \includegraphics[width=0.8\columnwidth]{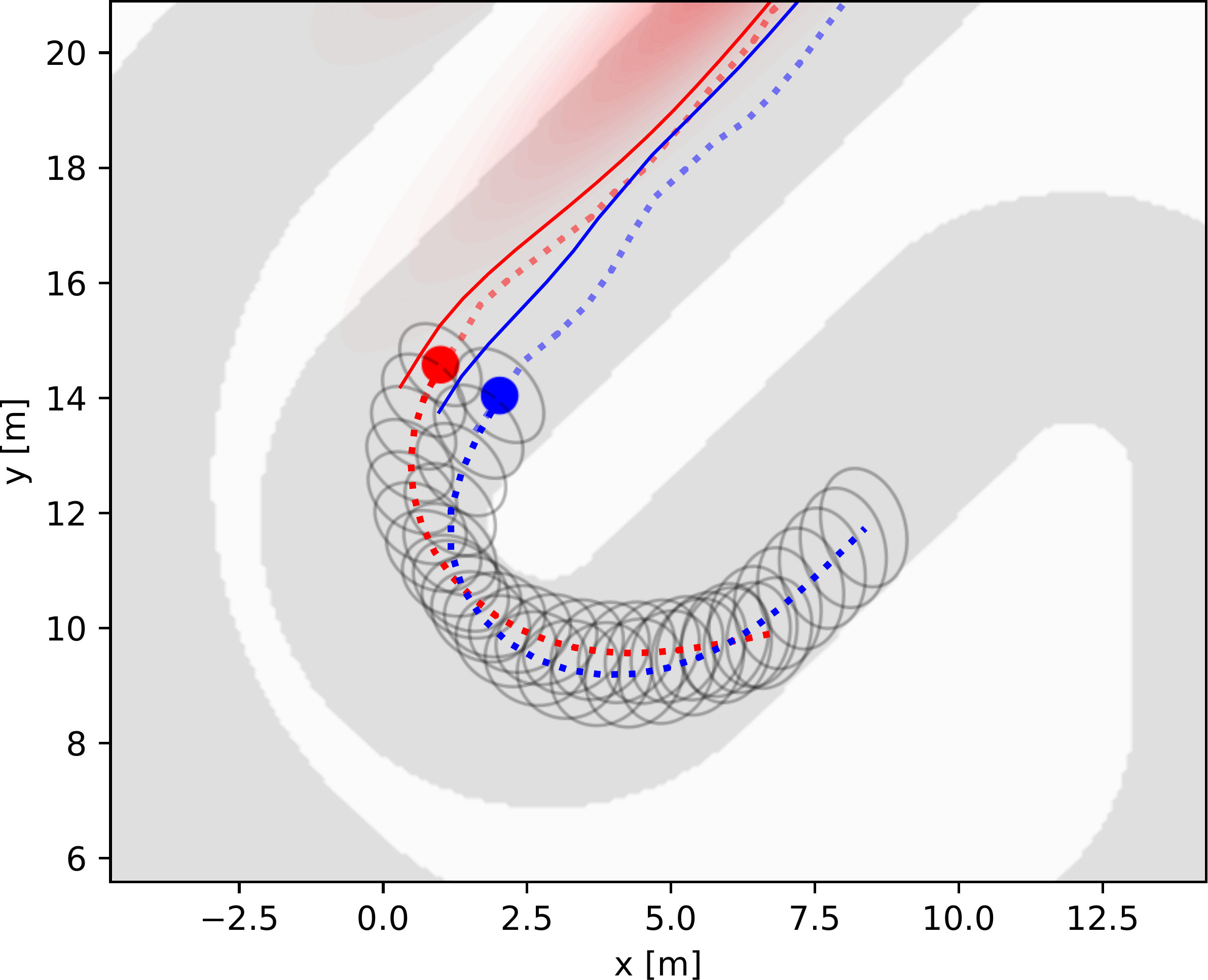}
        \includegraphics[width=0.8\columnwidth]{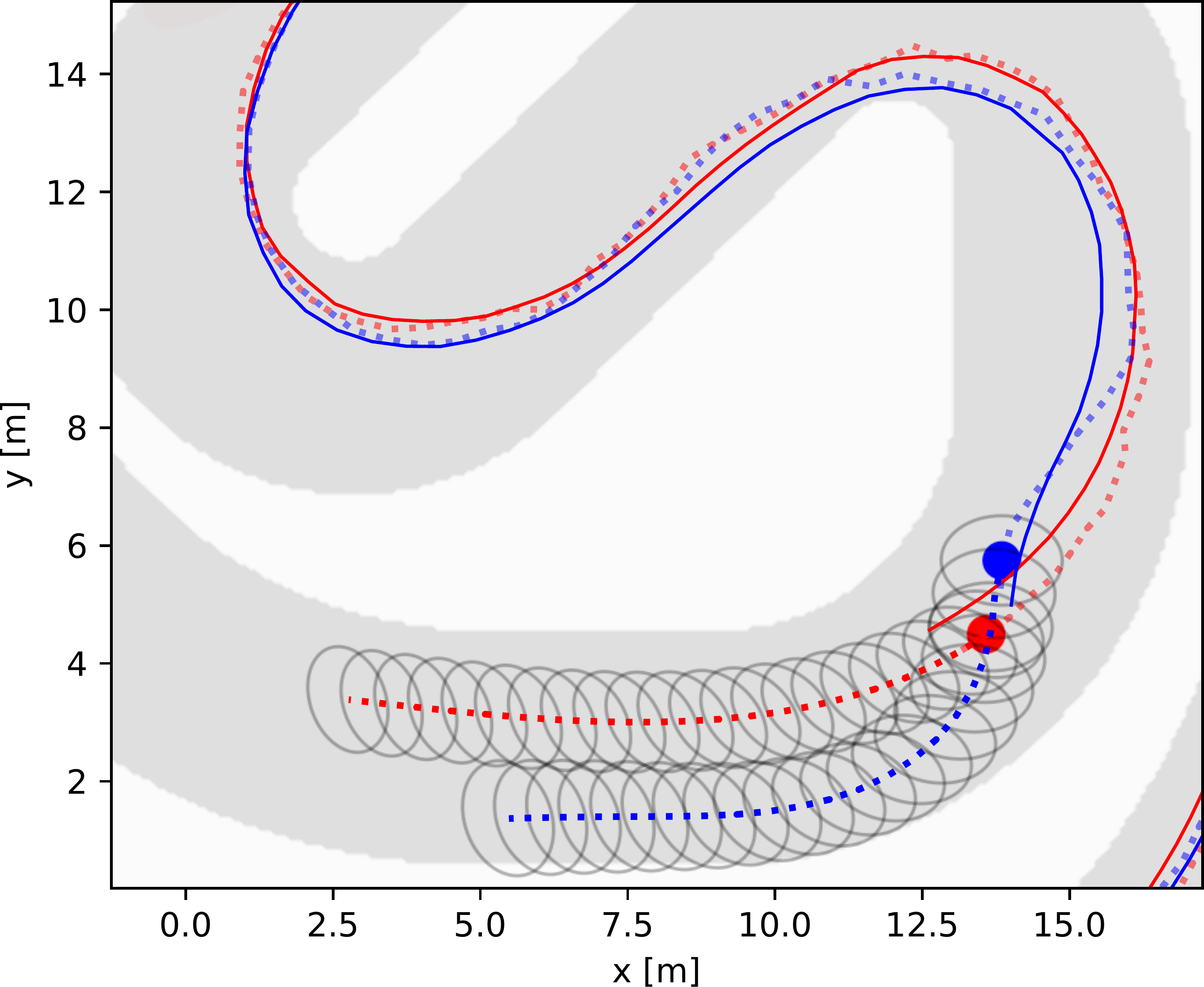}
    \caption{\textbf{Top:} Blue agent cuts in front of the red agent, forcing the red agent to break. As a result, the blue agent can remain in front of the red agent at the end of the turn.  \textbf{Bottom:} The red agent blocks the blue agent's overtaking maneuver forcing the blue agent to stay behind and take a wider line in the upcoming right turn. Significant amount of noise is simulated visualized by the deviation of the true trajectory (solid lines) and the predicted mean of the belief (dashed lines).}
    \label{fig:cutting_blocking}
\end{figure}

\subsubsection{Competitive Racing}
In our racing simulation, each car executes the current commanded control computed by their own separate instance of Algorithm~\ref{alg:nash_alg}. The environment's dynamics are propagated forward subject to significant amount of noise. Subsequently, a noisy observation is generated to simulate measurement uncertainty and the current belief is updated by an EKF step.
Each agent runs an individual and independent EKF and maintains their own separate belief over themselves and others. Each agent receives noisy measurements with noise drawn independently from other agents. Agents do not share any information, such as policies, measurements, initializations, state estimates, or beliefs, during online operation. To test robustness, we simulate substantial amounts of noise, such that the belief $\belief$ may significantly deviate from the true state of the system $\bx$, shown in Figure~\ref{fig:race_big} and Figure~\ref{fig:cutting_blocking}.

We encourage interaction by starting one agent with lower drag coefficient (and therefore higher speed) behind another \emph{slower} agent. The \emph{faster} agent will eventually catch up to the previous agent and initiate an overtaking maneuver. The better interactions are predicted and integrated into planning, the more successful overtaking maneuvers will occur.

The algorithm described in this paper is able to synthesize competitive emergent behavior such as blocking of other vehicles and cutting in front of others, illustrated in Figure~\ref{fig:cutting_blocking}. Additionally, although tight racing lines cut corners very closely, the chance constraints are successful in prohibiting collisions under the presence of motion and observation noise.

\begin{figure}
    \centering
        \includegraphics[width=0.9\columnwidth]{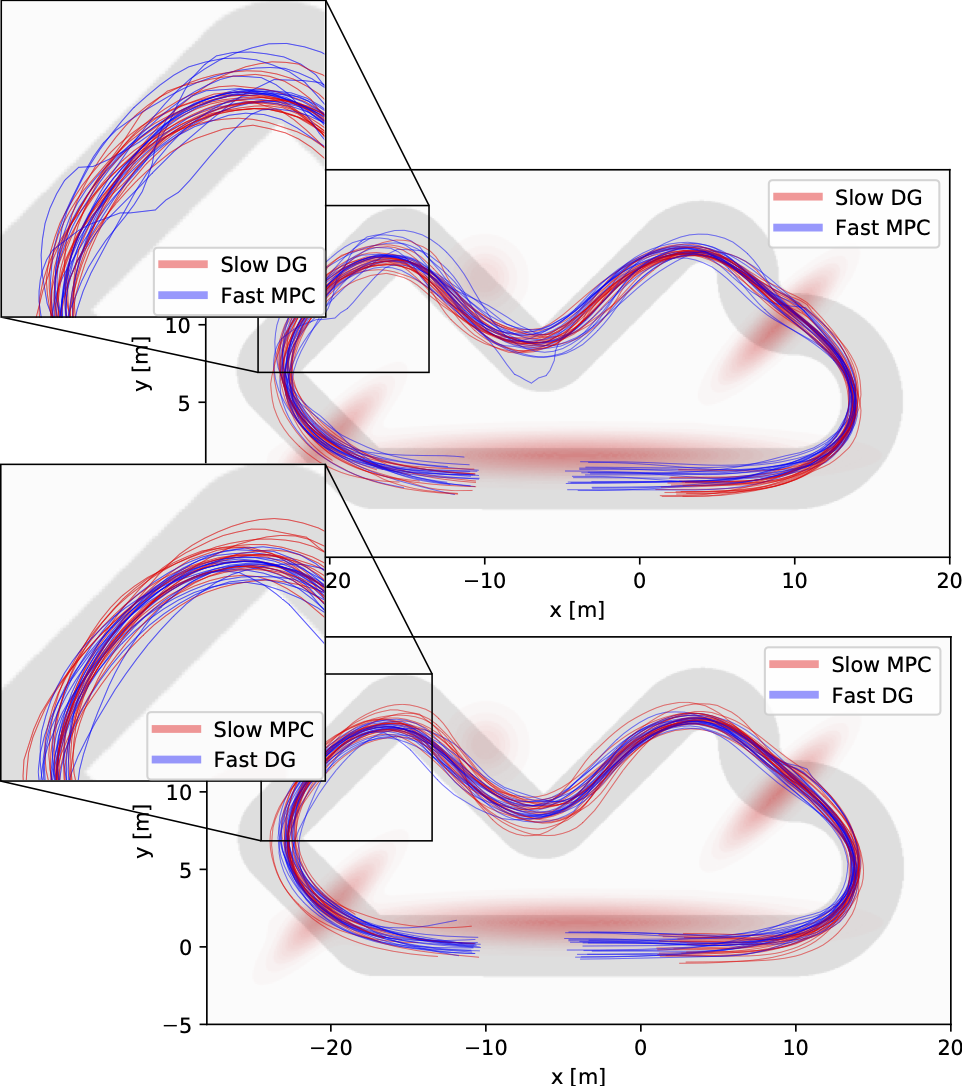}  
        \vspace{0.4cm}
		\includegraphics[width=\columnwidth]{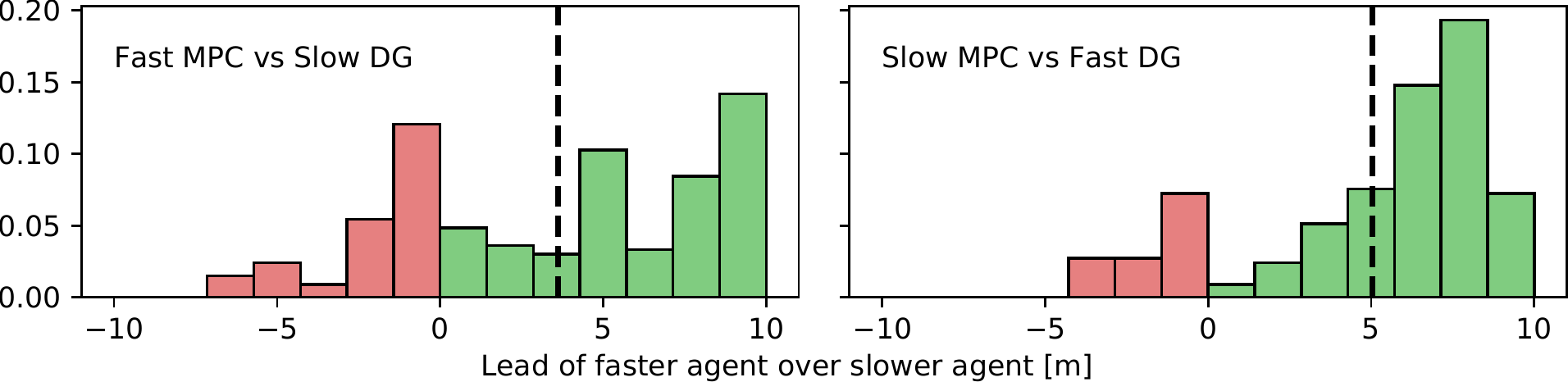}
    \caption{\textbf{Top:}
    Traces of agents comparing MPC and DG. In both cases the MPC method moves away from the ideal racing line more often due to failed overtaking attempts. It can not foresee it's influence on the DG agent's actions and thus is less efficient. It is also not able to take advantage of estimating is implicit control over the other agent like the DG agent. The agents start from random initialization locations around the origin. \textbf{Bottom:} Histograms of the $\Delta$arc-length lead of the faster agent over the slower agent. Green indicates that the faster agent won the race against the agent starting in the lead, whereas red indicates the opposite. In comparison, the DG method won more races than the MPC method and had a higher average lead.}
    \label{fig:DG_vs_MPC}
\end{figure}

\subsubsection{Benefits of Dynamic Game Planning}
We compare the performance of Dynamic Game (DG) planning to conventional methods such as Model Predictive Control (MPC). Both DG and MPC agents plan in belief space. The MPC agent has the exact same cost structure, but observes the other agents' executed actions and predicts agents to continue with the same action. The MPC baseline therefore predicts agents to not react to changes of their own actions and  cannot leverage the effects of their own actions on other agents.
The MPC is capable of synthesizing competitive racing trajectories, shown in Figure~\ref{fig:DG_vs_MPC}, which are identical to the DG trajectories when no other agents are present. The performance of the DG planning distinguishes itself when interactions occur. 
\begin{table}
\centering
\caption{Racing Performance: DG vs MPC and BSP vs non-BSP }
\begin{tabular}{l ||c}
Competition Pair                              & Fraction of Fast winning \\ \hline
Fast DG BSP vs Slow MPC BSP     & \textbf{82\%}         \\
Fast MPC BSP vs Slow DG BSP     & 64\%         \\ \hline
Fast DG BSP vs Slow DG non-BSP & \textbf{77\%}         \\
Fast DG non-BSP vs Slow DG BSP & 63\%        
\\ \hline
Fast DG BSP vs Slow DG BSP &\textbf{67\%}       \\
\end{tabular}
\label{tab:DG_vs_MPC}
\end{table}

\begin{table}
\centering
\caption{Racing Performance: Winning Ratio}
\begin{tabular}{l ||c}
Competition Pair                              & Win ratio \\ \hline
DG BSP vs MPC BSP    & \textbf{1.44:1}      \\ \hline
DG BSP vs DG non-BSP & \textbf{1.33:1}         \\
\end{tabular}
\label{tab:win_ratio}
\end{table}

We display the results of 200 runs in Figure~\ref{fig:DG_vs_MPC}, Table~\ref{tab:DG_vs_MPC} and Table~\ref{tab:win_ratio}. The DG method wins $44\%$ more races relative to the MPC baseline and has a larger lead on average. These results clearly illustrate the competitive advantage of our game-theoretic algorithm from leveraging how others react to one's own actions when planning.
\begin{figure}
\centering
        \includegraphics[width=\columnwidth]{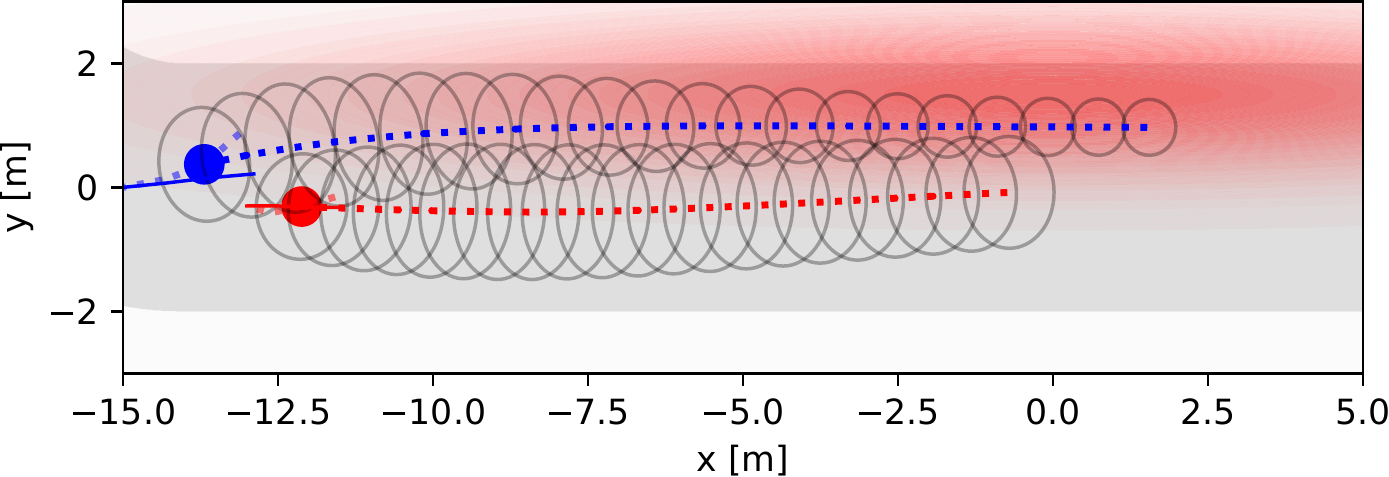}
        
\includegraphics[width=\columnwidth]{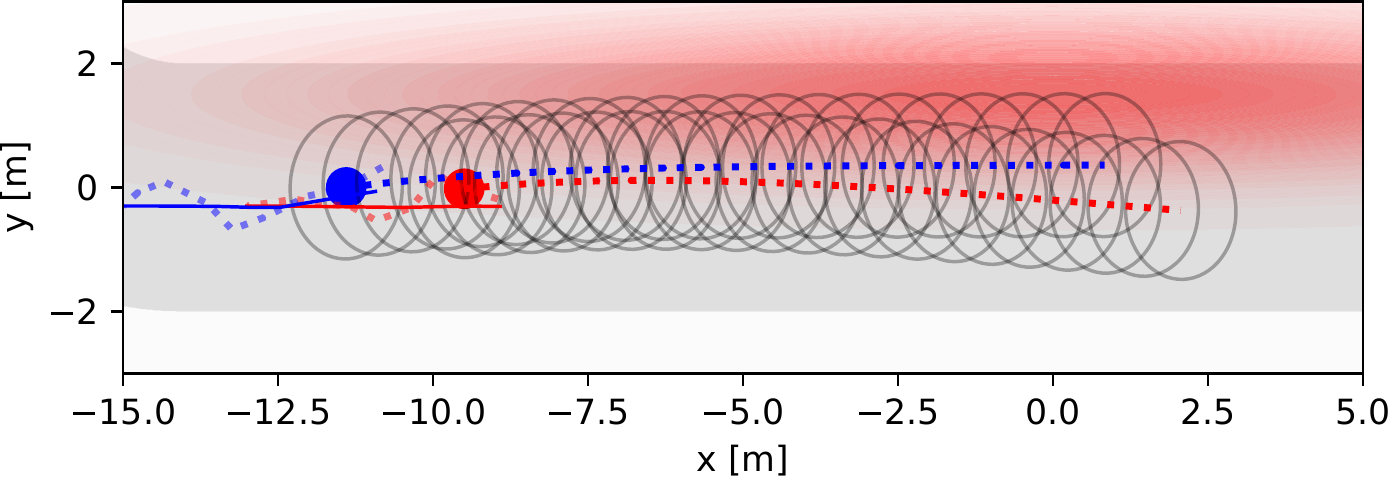}
    \caption{
    \textbf{Top:} The blue agent overtakes the red agent by decreasing the uncertainty through the low noise region and reducing the chance constraint (ellipses). \textbf{Bottom:} The blue agent has the same uncertainty over the planning horizon and fails to overtake since the chance constraints remain large.
    }
    \label{fig:overtake_info_gain}
\end{figure}

\subsubsection{Benefits of Belief-Space Planning}
We also compare the performance of DG planning with and without Belief-Space Planning (BSP). In the non-BSP case the current uncertainty $\Sigma_0$ of the belief $\belief_0$ is held constant over the planning horizon and is not influenced by expected measurements. Note that the current belief is still updated online by an EKF for both agents. Results are reported in Figure~\ref{fig:BS_vs_non_BSP} and Table~\ref{tab:DG_vs_MPC}. The BSP variant wins $33\%$ more races, has a larger average lead, and the fewest number of collisions. The non-BSP method collides nearly 10 times more often and exhibits behavior inappropriate for observed uncertainty levels. For example, agents are too conservative because low noise regions are not considered in the planning phase, or too aggressive when entering sharp turns since additional motion noise due to breaking and steering are not accounted for.

Figure~\ref{fig:overtake_info_gain} gives an intuitive explanation for the competitive advantage of planning in belief space. Without information gain, the follower will never be able to overtake due to the large chance constraint. Whereas with information gain, the chance constraint shrinks while moving through a low noise zone, allowing the blue agent to overtake the leading agent.
As shown in Figure~\ref{fig:BS_vs_non_BSP}, the BSP agent can adapt their trajectories to account for increased noise due to strong actuation, i.e. braking and steering, and gaining information in low noise regions.

Finally, we compare the performance when both agents use DG BSP, Table~\ref{tab:DG_vs_MPC}, and see that the faster agent wins 67\% of the races. Since the performance gain of the faster over the slower agent is smaller than in the previous two comparisons, we assume that the slower agent improves their blocking behavior more than the faster agent improves their ability to overtake. 
In scenarios where high uncertainty causes the chance constraints to occupy large parts of the track's width, the slower agent can often block the faster agent by proceeding in the middle of the road.

\begin{figure}
    \centering
    \includegraphics[width=0.95\columnwidth]{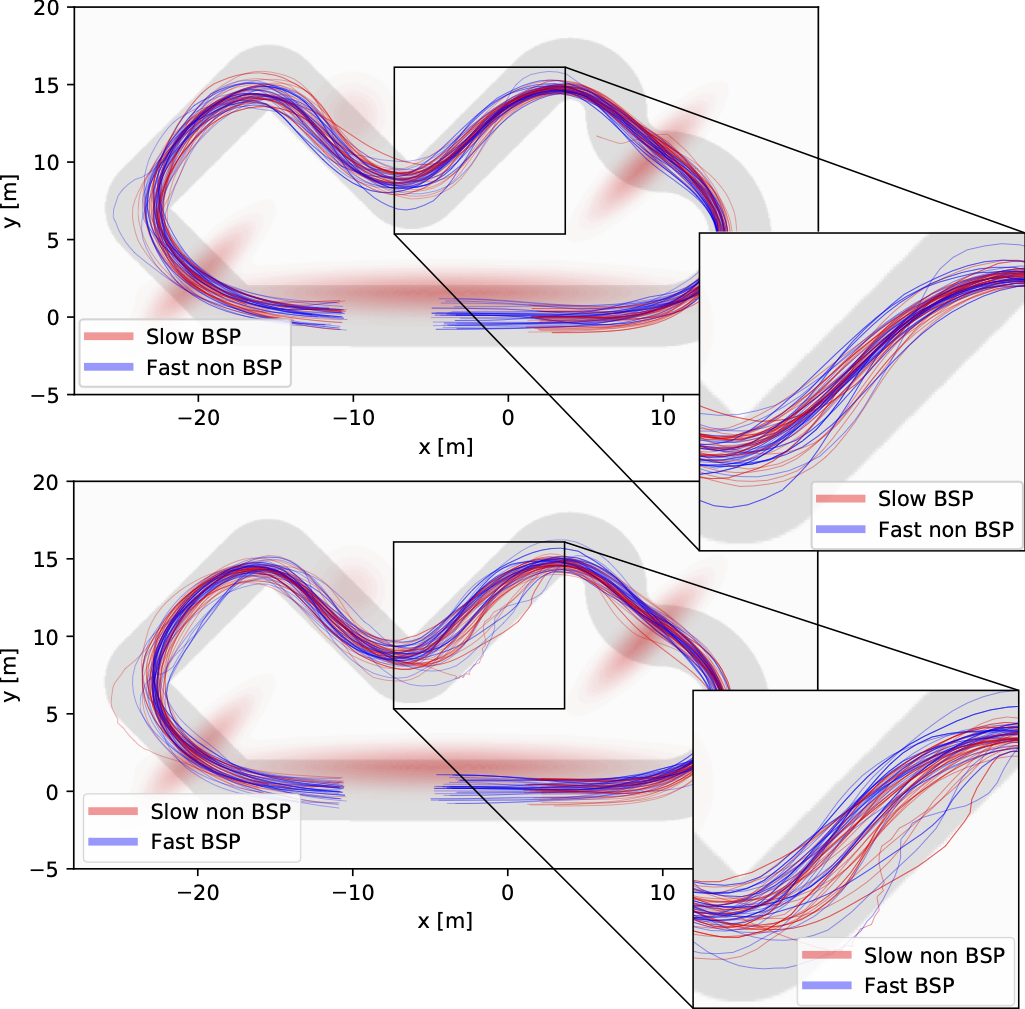}
    \vspace{0.4cm}
    \includegraphics[width=\columnwidth]{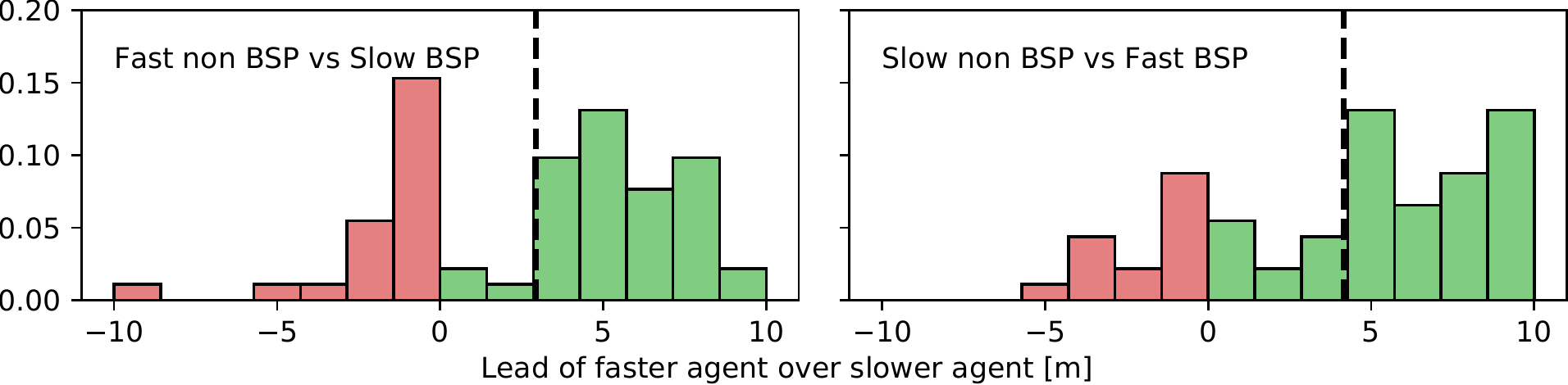}
    \caption{\textbf{Top:}
    Traces of agents comparing BSP and non-BSP. In both cases the non BSP method shows more unsafe behavior, leaving the track several times and nearly colliding with the other agent. The BSP agent attempts more aggressive overtaking maneuvers due to the lower uncertainty estimate over itself and the other agent, as shown in the cutout. The agents start from random initialization locations around the origin.
\textbf{Bottom:} Histograms of the $\Delta$arc-length  lead of the faster agent over the slower agent. Here, the BSP method won more races than the non-BSP method and had a higher average lead.}
    \label{fig:BS_vs_non_BSP}
\end{figure}

\subsection{Real-Time Implementation Details}
We implement our solver in the CasADi~\cite{Andersson2018} framework leveraging auto-differentiation by source code transformation, automatic problem specific compute graph generation, C-code generation, and sparse operations. Exploiting sparsity is highly important to allow for real-time performance since the belief space, encompassing the mean state and the upper triangle of the covariance matrix can make respective Jacobian and Hessian matrices very large. The average compute times on a Ryzen 7 1700X 3.4 GHz are reported in Table~\ref{tab:compute_time}. Algorithm 1 was run until convergence starting from a cold start for all experiments, i.e. the initial control trajectory $\bu$ consists of all zeros. Nonetheless, it is also possible to run the algorithm sequentially by hot starting the optimization with the previous solution. This is common practice in related optimization techniques for controls such as sequential quadratic programming~\cite{nocedal2006numerical} and allows to run Algorithm~\ref{alg:nash_alg} at 100-200Hz. In these cases it is often enough to run only very few iterations to update the previous solution.

\begin{table}[h]
\centering
\caption{Average computation time}
\begin{tabular}{l ||c  c  c}
Experiment          & Per iteration & Until convergence\\ \hline
Active surveillance & 9.3 ms         & 371.3 ms\\
Guide dog           & 11.2 ms        & 474.9 ms\\ 
Racing              & 5.8 ms         & 110.5 ms\\
\end{tabular}
\label{tab:compute_time}
\end{table}

\section{Conclusions} 
In this paper, we propose a formulation for integrating belief-space planning into dynamic games, and present a real-time algorithm for solving the local Nash equilibria of these dynamic games in belief space.
We demonstrate its performance of combining game-theoretic planning and information gathering with three case studies: active  surveillance, guiding blind agents, and racing with autonomous vehicles. 

While game-theoretic planning models the interaction and
dependency among agents, it does not address the quality of
information available to the agent for decision making.
Incorporating belief-space planning in dynamic games allows for new capabilities not possible with other approaches, essential in house service robots or interacting with human agents in traffic. Reasoning about another agent's uncertainty and simultaneously leveraging the effect of own actions on other agents' actions results in complex emergent behavior such as indirectly pushing and guiding others through regions of light, without the use of any form direct of communication.

In competitive use cases such as racing, emergent behavior consists of cutting, blocking, forcing others to break hard with the goal of increasing their uncertainty and slowing them down in turns, as well as the exploitation of high information-gain zones for overtaking.
In particular, game-theoretic belief-space planning significantly increased performance in dynamic racing when benchmarked against state-of-the-art planning methods.
Game-theoretic belief-space planning wins 44\% more races when competing with a non-game-theoretic baseline with belief-space planning and 34\% more races than a game-theoretic baseline without belief-space planning.

In this work we limit ourselves to first-order beliefs to avoid the explosion in parameters for recursive beliefs over beliefs. 
Nonetheless, even in cases where a first-order belief assumption is
a simplification of the true belief dynamics, such as racing, we see improved performance to baselines that do not take the belief over other agents into account.
In future work we intend to develop extensions beyond first-order belief spaces.

We efficiently solve for Nash equilibria in belief space and achieve real-time performance, operating our algorithm at more than 100Hz. Efficiently solving a quadratic game at each stage of the recursive backward pass of a belief-space variant of \ac{iLQG} results in an algorithm with runtime complexity $\mathcal{O}(l N^7 n_\bx^{i,6})$. Linear complexity in the planning horizon allows for online deployment, in comparison to point-based POMDP algorithms with exponential complexity. While our algorithm also achieves polynomial runtime complexity in the number of agents $N$, future work will investigate lowering the complexity further.

\ifCLASSOPTIONcaptionsoff
  \newpage
\fi

\balance
\bibliographystyle{IEEEtran}
\bibliography{references}

\end{document}